\documentclass[twoside,11pt,reqno]{article}
\usepackage{jmlr2e} 
\usepackage{amsmath, amssymb}
\usepackage{multirow,multicol,enumerate,enumitem}
\RequirePackage{hypernat}
\newcommand\MYhyperrefoptions{bookmarks=true,bookmarksnumbered=true,
pdfpagemode={UseOutlines},plainpages=false,pdfpagelabels=true,
colorlinks=true,linkcolor={magenta},citecolor={blue},pagecolor={black},
urlcolor={black},
pdftitle={A generalized risk approach to path inference based on hidden Markov models},
pdfsubject={Typesetting},
pdfauthor={J\"uri  Lember},
pdfkeywords={risk, HMM, hybrid, interpolation, MAP sequence, Viterbi
  algorithm, symbol-by-symbol, posterior decoding}}

\usepackage[\MYhyperrefoptions,pdftex]{hyperref}
\def\P{{\bf P}}
\def\R{{\bar R}}
\def\U{{\bar U}}
\ShortHeadings{Generalized risk-based path inference in HMMs}{Lember and Koloydenko}
\firstpageno{1}

\usepackage{url}
\usepackage{breakurl}

\begin{document}
\sloppy
\title{A generalized risk approach to path inference based on hidden Markov models}%
\author{\name J\"uri Lember
\email juri.lember@ut.ee \\
       \addr Institute of Mathematical Statistics\\
       Tartu University\\
       J. Liivi 2-507, Tartu, 50409, Estonia\\       
       \AND
       \name Alexey A. Koloydenko \email alexey.koloydenko@rhul.ac.uk \\
       \addr Department of Mathematics\\
       Royal Holloway, University of London\\
       TW20 0EX, UK}

\editor{}
\maketitle
\begin{abstract}
Motivated by the unceasing interest in hidden Markov
models (HMMs), this paper re-examines hidden path inference in these models, 
using primarily a risk-based framework.  
While the most common {\em maximum a posteriori} (MAP), 
or Viterbi, path estimator and the {\em minimum error}, or {\em Posterior Decoder} (PD) 
have long been around, other path estimators, or decoders, have been 
either only hinted at or applied more recently and in dedicated applications generally unfamiliar to
the statistical learning community. Over a decade ago, however,
a family of algorithmically defined decoders aiming
to hybridize the two standard ones was proposed \citep{Brushe1998}.      
The present paper gives a careful analysis of this hybridization approach, identifies several problems and issues with
it and other previously proposed approaches, and proposes practical resolutions of those.  
Furthermore, simple modifications of the classical 
criteria for hidden path recognition are shown to lead to a new class of 
decoders. Dynamic programming algorithms to compute these decoders in
the usual forward-backward manner are presented.
A particularly interesting subclass of such estimators can be also viewed as hybrids of 
the MAP and PD estimators. Similar
to previously proposed MAP-PD hybrids, the new class is parameterized by a small number
of tunable parameters.   Unlike their algorithmic predecessors, 
the new risk-based decoders are more clearly
interpretable, and, most importantly,  work ``out-of-the box'' in practice, which is demonstrated
on some real bioinformatics tasks and data. Some further generalizations and applications are
discussed in conclusion.
\end{abstract}

\section{Introduction}
\begin{sloppypar} 
Besides their classical and traditional applications in signal
processing and communications \citep{Viterbi, BahlJelinek1974,
HayesCover82, Brushe1998} (cf. also further references in
\citep{HMM}) and speech recognition \citep{vanaraamat, jelinek0,
jelinek, mitEM2,  philips, IBM2002, raamat, Rabiner86, mitEM1,
philips2,  strom}, hidden Markov models have recently become
indispensable in computational biology and bioinformatics
\citep{BurgeKarlin1997, BioHMM2, natHMM, BioHMM1, Brejova2008, MajorosOhler2007} as well as in
natural language modeling \citep{ManningStatNatLang, VogelStatTranslate} 
and  information security \citep{HMM4Crypto}.
\end{sloppypar}

At the same time, their spatial extensions, known as hidden Markov random field
models (HMRFM), have been immensely influential in spatial
statistics \citep{BesagGreen1993, SpatialHMM, HMRFGemanS, TitteringtonVarBayesHMMRF2009}, and
particularly in image analysis,  restoration, and segmentation
\citep{besag1986, gemans, gray2, Marroquin2003, Winkler2003}.
Indeed, hidden Markov models have been called `one of the most successful
statistical modeling ideas that have [emerged] in the last forty
years' \citep{HMM}.

HM(RF)Ms owe much of their success on the one hand to the
penetration of the Markov property from the hidden
layer to the posterior distribution, and  on the other, to the
richness of the observed system \citep{HMRFGemanS}. In
other words, in addition to the prior, the posterior distribution
of the hidden layer also possesses a Markov property (albeit
generally inhomogeneous even with homogeneous priors), whereas the
marginal law of the observed layer can still include global, i.e.
non-Markovian, dependence.

The Markov property of the posterior distribution and the
conditional independence of the observed variables given the hidden
ones, have naturally led to a number of computationally feasible
methods for inference about the hidden realizations as well as
model parameters.  HMMs are naturally a special case of {\it graphical 
models} \citep{lauritzen}, \citep[ch. 8]{Bishop2006}.

HMMs, or one dimensional HMRFMs, have been particularly popular not
least due to the fact that the linear order of the indexing set (usually
associated with time) makes exploration of hidden realizations
relatively straightforward from the computational viewpoint. In
contrast, higher dimensional HMRFMs generally require approximate,
possibly stochastic, techniques in order to compute optimal
configurations of the hidden field \citep{pickardfields1993,
jiali2006, Winkler2003, TitteringtonVarBayesHMMRF2009}.  In particular, a {\em maximum a posteriori}
(MAP) estimator of the hidden layer of an HMM is efficiently and
exactly computed by a dynamic programming algorithm bearing the name
of Viterbi, whereas a general higher dimensional HMRFM would
commonly employ a simulated annealing type method \citep{gemans,
Winkler2003} to produce approximate solutions to the same task.

There are also various useful extensions of the ordinary HMM, such as variable duration
semi-Markov models, and factorial HMMs, etc.  \citep[ch. 13]{Bishop2006}.  
All of the material in this paper is applicable to those extensions in 
a straightforward way. However, to simplify the exposition we focus below
on the ordinary HMM.

\subsection{Notation and main ingredients}
We adopt the machine and statistical learning convention, referring to the hidden and observed processes as $Y$ and $X$,
respectively, in effect reversing the convention that is more
commonly used in the HMM context. Thus, let $Y=\{Y_t\}_{t\ge 1}$ be
a Markov chain  with state space $S=\{1,\ldots,K\}$, $K>1$, and
initial probabilities $\pi_s=P(Y_1=s)$, $s\in S$.  
Although we include inhomogeneous chains in most of what follows, for
brevity we will still be suppressing the time index wherever this
does not cause ambiguity. Hence, we write
$\mathbb{P}=(p_{ij})_{i,j\in S}$ for all transition matrices. Let
$X=\{X_t\}_{t\geq 1}$ be a
process with the following properties. First, given $\{Y_t\}_{t\ge 1}$,  
the random variables $\{X_t\}_{t\ge 1}$ are conditionally
independent. Second, for each $t=1,2,\ldots$, the distribution of $X_t$ depends on 
$\{Y_t \}_{t\ge 1}$ (and $t$) only through $Y_t$. The process $X$ is sometimes called the 
{\it hidden Markov process} (HMP) and the pair $(Y,X)$ is referred to as a {\it
hidden Markov model} (HMM). The name is motivated by the assumption
that the process $Y$ (sometimes called a {\it regime}) is
generally non-observable. The conditional distribution 
 of $X_1$ given $Y_1=s$ 
is called an {\it emission distribution}, written as $P_s$, $s\in S$. 
We shall assume that the
emission distributions are defined on a measurable space 
$({\mathcal X},{\mathcal B})$, where ${\mathcal X}$ is usually $\mathbb{R}^d$
and ${\mathcal B}$ is the corresponding Borel $\sigma$-algebra. 
Without loss of
generality, we assume that the measures $P_s$ have densities
$f_s$ with respect to some reference measure $\lambda$, such 
as the counting or Lebesgue measure.

Given a set ${\mathcal A}$, integers $m$ and $n$, $m<n$, 
and a sequence $a_1,a_2,\ldots \in {\mathcal A}^\infty$,
we write $a_m^n$ for the subsequence $(a_m,\ldots,a_n)$. 
When $m=1$, it will be often suppressed. Thus,  $x^T:=(x_1,\ldots,x_T)$ and
$y^T:=(y_1,\ldots,y_T)$ stand for the fixed observed and unobserved
realizations, respectively, of the HMM $(X_t,Y_t)_{t\ge 1}$ up to
time $T\ge 1$. Any sequence $s^T\in S^T$ is called a \emph{path}.
This parallel notation (i.e. $s^T$ in addition to $y^T$) 
is necessitated largely by our forthcoming discussion of various loss functions, which
do require two arguments.   
We shall denote  the joint probability density
of $(x^T,y^T)$ by $p(x^T, y^T)$, i.e.
$$p(x^T, y^T):={\mathbf
P}(Y^T=y^T)\prod_{t=1}^Tf_{y_t}(x_t).$$ 
To make mathematical expressions more compact, we overload 
the notation when this causes no ambiguity. Thus,
$p(s^T)$ stands for the probability mass function ${\mathbf P}(Y^T=s^T)$ of 
path $s^T$, and $p(x^T)$ stands for the (unconditional) probability density 
function $\sum_{s^T\in S^T}p(x^T,s^T)$ of the observed data $x^T$. 
It is standard (see, e.g. \citep{HMP, HMM}, \citep[ch. 13]{Bishop2006}) in this context to define
the so-called {\it forward} and {\it backward} variables
\begin{eqnarray}\label{eqn:fwv}
\alpha_t(s):=p(x^t|Y_t=s)P(Y_t=s),&
\beta_t(s):=\left\{ \begin{array}{ll} 1, & \hbox{if $t=T$} \\
  p(x_{t+1}^T|Y_t=s), & \hbox{if $t<T$}
 \end{array}
 \right.,&
\end{eqnarray}
where $p(x^t|Y_t=s)$ and $p(x_{t+1}^T|Y_t=s)$ are the conditional densities
of the data segments $x^t$ and $x_{t+1}^T$, respectively, given $Y_t=s$.

\subsection{Path estimation}\label{sec:segment}
Our focus here is estimation of the hidden 
path $y^T$. This task can also be viewed as 
{\it segmentation} of the data sequence into regions with distinct
class labels \citep{intech}. 
Treating $y^T$ as missing data \citep{tutorial}, 
or parameters, a classical and by far the most popular solution to 
this task is to maximize  
$p(x^T, s^T)$ in $s^T\in S^T$. Often, especially in
the digital communication literature \citep{ErrorControCoding,
Brushe1998}, $p(x^T, s^T)$ is called the {\it likelihood function} 
which might become potentially problematic in the presence of any genuine model parameters. 
Such ``maximum likelihood'' paths are also called {\em Viterbi paths} or {\em Viterbi alignments} 
after the Viterbi algorithm \citep{Viterbi, tutorial} commonly used for their computation. 
If $p(s^T)_{s^T\in S^T}$ is thought of as the
prior distribution of $Y^{T}$, then the Viterbi path also maximizes
$p(s^T|x^T):={\mathbf P}(Y^T=s^T|X^T=x^T)$, the probability mass
function of the posterior distribution of $Y^T$, hence the term `{\em
maximum a posteriori (MAP) path}'.

In spite of its computational attractiveness, inference based on the Viterbi paths may
be unsatisfactory for a number of reasons, including its
sub-optimality with regard to the number of correctly estimated
states $y_t$. Also, using the language of information theory, there is no
reason to expect a Viterbi path to be
{typical} \citep{AVTproof}. Indeed, ``there might be many similar paths through   
the model with probabilities that add up to a higher 
probability than the single most probable path'' \citep{KallKrogh2005}. 
The fact that a MAP estimate need not be representative of the posterior distribution
has also been recently discussed in a more general context in
\citep{CentroidEstimators2008}.  Atypicality of Viterbi paths particularly concerns
situations when estimation of $y^T$ is  combined with inference about 
model parameters, such as the transition probabilities $p_{ij}$ \citep{AVTproof}. Even when
estimating, say, the probability of heads from independent tosses of
a biased coin, we naturally hope to observe a typical realization
and not the constant one of maximum probability.

An alternative and  very natural way to estimate $y^T$ is by
maximizing the posterior probability $p_t(s|x^T):={\mathbf P}\left(Y_t=s|X^T=x^T\right)$ 
of each individual hidden state $Y_t$, $1\le t\le T$ \citep{BahlJelinek1974}. We refer to the corresponding 
estimator as {\it pointwise maximum a posteriori (PMAP)}. 
PMAP is well-known to maximize the expected number of correctly estimated
states (Section \ref{sec:risk}), hence the characterization `{\it optimal accuracy}' 
\citep{HolmesDurbin98}.
In statistics,
especially spatial statistics and image analysis, this type of
estimation is known as {\em Marginal Posterior Mode}
\citep{Winkler2003} or {\em Maximum Posterior 
Marginals} \citep{RueLossMPM1995} (MPM) estimation.  
In computational biology, this is also known as the {\it posterior decoding} (PD) \citep{Brejova2008}
and has been reported to be particularly successful in pairwise sequence alignment
\citep{HolmesDurbin98} and when more than one path has its posterior probability as ``high'' or nearly as
``high'' as that of the Viterbi path \citep{natHMM}.
In the wider context of biological applications of
discrete high-dimensional probability models,
this has also been 
called {\it consensus} estimation, and in the absence of constraints, {\it centroid} estimation 
\citep{CentroidEstimators2008}. In communications
applications of HMMs, largely  influenced by the BCJR algorithm of 
\citep{BahlJelinek1974},  the terms `{\em optimal symbol-by-symbol
detection}' \citep{HayesCover82}, `{\em symbol-by-symbol MAP
estimation}' \citep{Robertson95optimaland}, and `{\em MAP state
estimation}' \citep{Brushe1998} have been used for this.  Remarkably,
even before observing the data, optimal accuracy (i.e. based on the prior
instead of the posterior distribution) decoding can still be more
accurate than the Viterbi decoding (see subsection \ref{sec:expsummary})! 

\subsubsection{How different are PMAP and MAP inferences 
and how much room is in between the two?}\label{sec:pathdiffers}
This is a natural question in both practice and theory, 
especially for anyone  interested in 
seeking to improve performance of applications based on these methods while
maintaining their computational attractiveness. 

The belief that the 
difference between PMAP and Viterbi inferences is negligible may 
in part be explained by the concluding remark made in \citep{BahlJelinek1974} in the
special context of linear codes: ``Even though Viterbi decoding is not
optimal in the sense of bit error rate, in most applications of
interest the performance of both [PMAP and Viterbi] algorithms would be effectively
identical.'' This conclusion may also be 
explained by the dominance of binary chains in the telecommunication 
applications, and the binary state space indeed leaves too little room for the two inferences to differ. 
However, as HMMs with larger state spaces gained more prominence, 
it became clear that there could be appreciable differences between the PMAP and Viterbi 
inferences.  In fact, already two decades later, \citep{Brushe1998} contemplate
hybridization of the PMAP and Viterbi decoders, writing 
``Indeed, there may be applications where a delicate performance 
dependence exists between [the Viterbi and PMAP] estimates.  In such
cases, the use of a hybrid scheme \ldots may result in performance gains.''
We return to their idea later in this paper.

We are not aware of any systematic comparisons of the PMAP and 
Viterbi decoders apart from the aforementioned observations. 
Soon after the first version of this article was posted on \texttt{arXiv}, however,   
\citep{YauHolmes2010arXiv} reported similar interests in the matter, supported
by real and simulated examples. 
Of course, 
it has long been well-known \citep{tutorial} that
despite being optimal in the sense of maximizing the expected number of
correctly estimated states, a PMAP path can at the same time
have very low, possibly zero, probability. Thus, 
on the logarithmic scale, the difference in path probabilities between 
the PMAP and Viterbi decoders can easily be {\it infinite}.  In subsection~\ref{sec:example},
we give a real data example with only six hidden states to show that 
besides the infinite difference in the log-probabilities,
the two decoders can differ by more than 13\% in accuracy (i.e. error rate). 

We first show (subsection \ref{sec:pvdmore}) that it is actually not difficult to constrain the PMAP decoder 
to {\it admissible} paths, where admissibility is defined relative to the posterior distribution. Specifically, given $x^T$, a path $y^T$ is called
{\it admissible if its posterior probability $p(y^T|x^T)$ is defined and positive}, i.e. if $p(x^T, y^T)>0$.
We then explain that constraining the PMAP decoder 
to the paths of positive prior probability, as already done by others (see more below), is  not sufficient 
(albeit necessary) for admissibility of the PMAP paths.
Note that in a slightly more general form
allowing for state aggregation,  \citep{KallKrogh2005} do exactly this, i.e. force PMAP paths to have
positive prior probability, referring to the result as ``a possible path through the model''.
Thus, \citep{KallKrogh2005} appear to ignore that having a positive prior probability 
is {\it not sufficient in general for a PMAP path to be ``a possible path through the model''}, unless,
of course, ``the model'' is to be understood as the hidden Markov chain only and not the whole HMM. 
We will refer to the PMAP decoder constrained to the
admissible paths   as the {\it admissibly constrained PMAP}, or, simply
 {\it constrained PMAP}. This also details and clarifies our earlier discussion of
admissibility in \citep[Section 2]{intech}, which, like \citep{tutorial, KallKrogh2005}, 
also ignored the distinction between {\it a priori} and {\it a posteriori} modes of admissibility.

A variation on the same idea of making PMAP paths admissible has  been 
applied in \citep{PMAP2005} for  prediction of membrane proteins,
giving rise to the \emph{posterior Viterbi decoding (PVD)} \citep{PMAP2005}.   
PVD, however, maximizes the product $\prod_{t=1}^Tp_t(s_t|x^T)$ \citep{PMAP2005} (and also 
\eqref{eq:remove0slog} below) and not the sum 
$\sum_{t=1}^Tp_t(s_t|x^T)$, whereas the two criteria 
are {\it no longer equivalent in the presence of path constraints} (subsection \ref{sec:pvdmore}).
While acknowledging this latter distinction between their decoder and PVD and
not distinguishing between the prior and posterior modes of admissibility, \citep{KallKrogh2005}
appear to be unaware of the other distinction between their decoder and PVD: PVD 
paths are guaranteed to be of not only positive prior probability
but also of positive posterior probability, i.e. admissible (in our sense of the term).
In \citep{HolmesDurbin98}, a PMAP decoder is proposed to obtain
optimal pairwise sequence alignments.  \citep{HolmesDurbin98} use the term 
``a legitimate alignment'' which suggests admissibility, but the description
of their algorithm \citep[Section 3.8]{HolmesDurbin98} appears to be insufficiently
detailed to verify if the output is guaranteed to be admissible, or only of positive prior
probability, or, if inadmissible solutions are altogether an issue in that context.

Our own experiments
(Section~\ref{sec:example}) show that both PVD and constrained PMAP
decoder can return paths of very low (posterior) probabilities.  
Moreover, in many applications, e.g. gene identification and protein secondary structure prediction, the pointwise (e.g. nucleotide level) 
error rate is not necessarily the main  measure of accuracy (see also subsection \ref{sec:moremotivation} below), hence the constrained PMAP 
need not be an ultimate answer in that respect either.  Together with the above problem of 
atypicality of MAP paths, this has been addressed by moving from single path inference
towards \emph{envelopes} \citep{HolmesDurbin98}.
Thus, for example, in computational biology a common approach would be to aggregate individual 
states into a smaller number of semantic labels (e.g. codon, intron, intergenic).  
In effect, this would realize the notion of path similarity 
by mapping many ``similar'' state paths to a single label path, or {\it annotation}
\citep{Krogh97twomethods, KallKrogh2005, PMAP2005, Brejova2008}.  However, since this
mapping would usually be many-to-one (what  \citep{Brejova2007} refer to as the ``multiple path problem''), 
the annotation of the Viterbi path would  generally be inferior to the optimal (i.e.~MAP) annotation.
On the other hand, to compute the MAP annotation in many
practically important HMMs can be NP-hard \citep{Brejova2007} (which is not surprizing given
that the coarsened hidden chain on the set of labels is generally no longer Markov). Unlike the 
Viterbi/MAP decoder, the
PMAP decoder, owing it to its symbol-by-symbol nature, handles annotations as easily as it does state paths,
including the enforcement of admissibility. Interpreting admissibility relative to the prior distribution, 
this was shown in \citep{KallKrogh2005}, and this paper extends their result to  
admissible (i.e. relative to the posterior probability) paths and 
indicates further extensions in Section \ref{sec:discussion}.

A number of alternative heuristic
approaches are also known in computational biology,
but none appears to be fully satisfactory \citep{Brejova2008}.   
Overall, although the original Viterbi decoder has still been the most popular paradigm in 
many applications, and in computational biology in particular, alternative approaches have often 
demonstrated significantly better performance, e.g., in predicting various biological features.
For example, \citep{Krogh97twomethods} suggested the {\it $1$-best} algorithm for optimal labeling.  
More recently, \citep{PMAP2005} have demonstrated PVD to be superior to the $1$-best
algorithm, and, not surprisingly, to the Viterbi and PMAP decoders, on tasks of predicting membrane proteins.

Thus, a starting point of this contribution was that restricting the PMAP decoder to admissible paths 
is but one of {\it numerous ways to combine the strong points of the MAP and PMAP path estimators}.
Indeed, the popular seminal tutorial \citep{tutorial}  briefly mentions maximization of 
the expected number of correctly decoded (overlapping) blocks of length two or three, 
rather than single states as a sensible remedy 
against vanishing probabilities (albeit leaving it unclear if 
prior or posterior probability is meant).  With $k\ge 1$ and $\widehat y^T(k)$ being the block length 
and corresponding path estimate, respectively, this approach 
yields Viterbi inference as $k$ increases to $T$ (with $\hat y^T(1)$ corresponding to 
PMAP). Therefore, this could be interpreted as discrete interpolation 
between the PMAP and Viterbi inferences.  Intuitively, following Rabiner's logic,
one might also expect  $p(x^T, \widehat y^T(k))$ to  
increase with $k$. However, {\it this is not true} and it is possible for 
the  decoder $\hat y^T(2)$ to  produce an inadmissible (with the prior probability
being also zero) path
while the PMAP path is admissible:
$p(x^T, \widehat y^T(2))=0=p(\widehat y^T(2))<p(x^T, \hat y^T(1))$.  {\it We are not aware 
of this observation being previously made in the literature.}
Moreover, our experiments in 
Section~\ref{sec:example} show that this situation is far from being uncommon. 

On a related note, concerned with the same deficiencies of
the MAP and PMAP inferences, \citep{YauHolmes2010arXiv} have most recently 
also used the decision-theoretic framework to allow for full asymmetry in the othewise symmetric 
pairwise loss (see \eqref{eq:dnew} with $k=2$ below) that underpins the  decoder $\hat y^T(2)$.
This is no doubt a very natural extension to provide to the end user, and (partially) 
asymmetric pairwise losses  had indeed been incorporated in a prominent webserver in the context of 
RNA secondary structure prediction  \citep{Sato01072009}.

Despite the possibility of $\hat y^T(2)$ or its asymetric syblings to return forbidden sequences, 
we find the idea of interpolation between the PMAP and
Viterbi inferences very interesting.  Besides \citep{YauHolmes2010arXiv} 
acknowledging the need for intermediate modes of inference, to the best of our knowledge, the only
published work that explicitly proposes such an interpolation is
\citep{Brushe1998}. Their approach  is algorithmic, which makes 
it difficult to interpret its paths in general and analyze their properties (e.g. asymptotic
behavior). More importantly, \citep{Brushe1998} claim that the family of their
interpolating decoders can work in practice, which, as we explain in detail in Section 
\ref{subsec:alg}, is not true apart from trivial situations.  
Despite these and other deficiencies of their approach, it raises some interesting
questions and inspires interesting modifications, which we also discuss in Section \ref{subsec:alg}.

\subsubsection{Further motivation}\label{sec:moremotivation}
One other motivation for considering new decoders is that unlike the error rate or path 
probability, analytic optimization of
other performance measures (e.g. Matthew's correlation 
\citep{protein_secondary_struct_predict_2006_borodovsky}, $Q_2$, $Q_{ok}$, SOV \citep{PMAP2005}, etc.) 
used in practice is difficult if at all possible. Having a large family of computationally 
efficient decoders, such as the new generalized hybrid decoders, and using some training data, 
one can select empirically a member from the family 
that optimizes the performance measure of interest. More generally, it seems
 advantageous for applications to be aware of the new choices of decoders and their properties.

Also, depending on the application, the emphasis sometimes shifts from  purely automatic
decoding with hard decisions to data exploration. It is then particularly valuable
to gain insights into the topology of the state space in the sense of 
identifying compartments of high concentration of the posterior distribution.
The significance of identifying clusters (of similar sequences) of high (total) posterior probability
in high-dimensional discrete spaces has been recently discussed in \citep{CentroidEstimators2008},
and a  thorough discussion of the advantages of topological and geometric approaches to analysis of 
complex data in general has more recently appeared in \citep{GCarlssonTopologyAndData2009}. 
In this context, it appears to be beneficial 
to output a family of related decodings instead of one or several (``$N$ best'') decodings that are 
optimal relative to a single criterion such as MAP. For instance, by ``smoothly'' varying the 
optimization criterion, saliency of detections of interesting features
can be assessed and a better understanding of a  neighborhood of solutions can be gained
(e.g. discerning between an ``archipelago'' and a ``continent''), all without having
to compute, or even  define explicitly, a path similarity measure (such as those based on BLAST scores).  
Used within this context, this relatively inexpensive type of ``neighborhood'' inference might become either 
alternative or complementary to sampling (from the posterior distribution); 
see also Section~\ref{sec:example} and Section~\ref{sec:discussion}.

\subsection{Organization of the rest of the paper}\label{sec:organ}
In this paper, we consider the path inference problem in the more
general framework of statistical learning. Namely, we consider sequence {\it classifier}
mappings 
$$g_t: {\mathcal X}^t \to S^t, \quad t=1,2,\ldots,$$ and optimality criteria for their selection.
When all $g_t$'s are obtained using the same decoding principle, or optimality criterion, 
we refer to them collectively as a {\it classification method}, or simply, {\it decoder}, $g$. 
Where this causes no ambiguity, we will simply be writing $g(x^t)$ instead of $g_t(x^t)$.
In Section \ref{sec:risk}, criteria for optimality of $g$ are naturally formulated 
in terms of risk minimization whereby $R(s^T|x^T)$, the {\it risk of} $s^T$, derives from a suitable
{\it loss function}.  In Section \ref{sec:combrisk}, we consider
families of risk functions which naturally
generalize those corresponding to the Viterbi and PMAP solutions 
(subsection \ref{subsec:viterbipmap}).  Furthermore, as shown in Section
\ref{sec:bridge}, these risk functions define a family of path
decoders parameterized by an integer $k$ with $k=1$ and $k\to\infty$ 
corresponding to the PMAP and Viterbi cases, respectively 
(Theorem \ref{k-block}). A  continuous mapping
via $k=1/(1-\alpha)$, $0\le \alpha\le 1$ compactifies this parameterization and 
further enriches the solution space by including fractional $k$. 
It is then discussed how the new family of decoders can be
embedded into yet a wider class with a principled criterion of optimality.
We also compare the
new family of decoders with the Rabiner $k$-block approach. 
Any decoder (classifier) would only be of theoretical
interest if it could not  be efficiently computed. In Section
\ref{sec:combrisk}, we show that all of the newly defined decoders can be
implemented efficiently  as a dynamic programming algorithm in the
usual forward-backward manner with essentially the same (computational as well
as memory) complexity as the PMAP or Viterbi decoders (Theorem
\ref{dyn}).  Recent advances in the asymptotic theory of some of the main decoders and risks presented in this paper 
is reviewed in Section \ref{sec:asy} together with sketches of how these may be relevant in practice. 
Various further extensions are discussed in the concluding Section \ref{sec:discussion}.

\subsection{Key contributions of the paper}\label{sec:contribution}
We review HMM-based decoding within the sound framework of statistical decision theory,
and do so notably more broadly than has been done before, e.g. in the prominent  work
of \citep{CentroidEstimators2008}.
We also investigate thoroughly previous work on combining the desirable properties
of the two most common HMM-based decoders, i.e. the Viterbi and optimal accuracy decoders. In doing so,
we discover several relevant claims and suggestions 
to be unjustified, misleading, or plainly incorrect.  We explain in detail those  
deficiencies, giving relevant counterexamples, and show how they can be resolved. 
Some such resolutions are naturally left within the native frameworks of the originals, 
whereas others are more naturally given within the general risk-based framework.
All of the resulting decoders are shown to be easily implementable within the usual forward-backward 
computational frameworks of the optimal accuracy and Viterbi decoders.
We argue that the richness, flexibility, and analytic interpretation of the resulting families of decoders 
offer new possibilities for applications and invite further theoretical analysis.

Specifically, this paper
\begin{enumerate}[label=\emph{\arabic*)},leftmargin=0cm,itemindent=.5cm,labelwidth=\itemindent,labelsep=0.1cm,align=left]
\item gives a clear definition of admissibility of hidden paths;
\item shows that optimal accuracy decoding when constrained to the paths of positive prior 
probability, can still return inadmissible paths;
 \item shows that  the suggestion of \citep{tutorial} to maximize the rate of correctly recognized
blocks does not work for blocks of size two and therefore can be misleading;
\item proposes suitable risk functions to ``repair'' the above suggestion, and thus 
designs new and rich families of computationally efficient decoders;
\item unifies virtually all of the key decoders within the same  risk-based
framework; 
\item establishes theoretical results regarding key properties of the new decoders, in particular
\item establishes a curious general result on convex decomposition of the key risk functionals for Markov chains;  
\item experimentally illustrates the newly proposed families of decoders
using real data;
\item explains how the idea of hybridization of the Viterbi and optimal accuracy
decoders given in \citep{Brushe1998} can fail when the Viterbi path is not unique;
\item establishes that the claims made in the same work regarding the implementation of  their algorithm  
to hybridize the Viterbi and optimal accuracy decoders are incorrect;
\item shows how the corresponding forward and backward variables given in the same work 
can be easily scaled to produce an operational decoding algorithm; 
\item shows that the resulting decoders are different from
the original hybrid decoders of \citep{Brushe1998};
\item proposes the more common power-transform to replace the somewhat idiosyncratic transform
of \citep{Brushe1998} to yield an immediately operational algorithm to hybridize the Viterbi and optimal 
accuracy decoders (at least when the Viterbi path is unique), allowing also for 
extrapolations ``beyond'' the optimal accuracy decoder;
\item indicates a number of further extensions of the new family of decoders.
\end{enumerate}

\section{Risk-based path inference}\label{sec:risk}
Given a sequence of observations $x^T$ with $p(x^T)>0$, we view the (posterior) {\it risk} 
as a function
$$R(\cdot|x^T):\quad S^T \mapsto [0,\infty].$$
Naturally,  we seek a state sequence with minimum risk: $g^*(x^T):=\arg\min_{s^T\in S^T}R\big(s^T|x^T\big)$.
In the {\em statistical decision and pattern recognition theories}, the
classifier $g^*$ is known as the {\it Bayes classifier}
(relative to risk $R$). Within the same framework, the risk is often
specified via a {\it loss-function} $$L: S^T\times
S^T \to [0,\infty],$$ interpreting $L(y^T,s^T)$ as the loss 
incurred by the decision to predict $s^T$ when the
actual state sequence was $y^T$. Therefore, for any
state sequence $s^T\in S^T$, the  risk is given by
\begin{equation*}\label{risk}
R(s^T|x^T):=E[L(Y^T,s^T)|X^T=x^T]=\sum_{y^T\in
S^T}L(y^T,s^T)p(y^T|x^T).
\end{equation*}
\subsection{Standard path inferences re-examined}\label{subsec:viterbipmap}
The most popular loss function is the so-called {\it symmetrical} or {\it zero-one} loss 
$L_{\infty}$ defined as follows:
\begin{equation*}\label{symm-n}
L_{\infty}(y^T,s^T)=\left\{
               \begin{array}{ll}
                 1, & \hbox{if $s^T\ne y^T$;} \\
                 0, & \hbox{if $s^T=y^T$.}
               \end{array}
             \right.
\end{equation*}
We shall denote the corresponding risk by $R_{\infty}$. With this
loss, clearly 
\begin{equation}\label{eqn:Rinf}
R_{\infty}(s^T|x^T)=\P(Y^T\ne s^T|X^T=x^T)=1-p(s^T|x^T),
\end{equation}
thus
 $R_{\infty}(\cdot|x^T)$ is minimized by a {Viterbi path}, i.e. a sequence of
maximum posterior probability. Let $v(\cdot ; \infty)$ stand
for the corresponding classifier, i.e.
$$v(x^T; \infty):=\arg\max\limits_{s^T\in S^T}p(s^T|x^T),$$ 
with a suitable tie-breaking rule.

Note that Viterbi paths also
minimize the following risk
\begin{equation}\label{loglikerisk}
\R_{\infty}(s^T|x^T):=-{1\over T}\log p(s^T|x^T).
\end{equation} It can actually be advantageous to use the logarithmic risk \eqref{loglikerisk}
since, as we shall see later,  this  leads to
various natural generalizations (Sections \ref{sec:combrisk} and
\ref{sec:bridge}).

When sequences are compared pointwise, it is common to 
use additive loss functions of the form 
\begin{equation}\label{point-loss}
L_{1}(y^T,s^T)={1\over T}\sum_{t=1}^Tl(y_t,s_t),
\end{equation} where $l(y_t,s_t)\geq 0$ is
the loss associated with classifying the $t$-th element $y_t$ as $s_t$. Typically,
for every state $s$, $l(s,s)=0$. It is not hard to see that, with
$L_{1}$ as in \eqref{point-loss}, the corresponding risk can be
represented as follows
\begin{equation*}
 R_1(s^T|x^T)={1\over T}\sum_{t=1}^T R_t(s_t|x^T),
\end{equation*}
where $R_t(s|x^T)=\sum_{y\in S} l(y,s)p_t(y|x^T)$. Most
commonly, $l$ is again symmetrical, or zero-one, i.e. $l(y,s)=\mathbb{I}_{\{s\ne y\}}$,
where  $\mathbb{I}_A$ stands for the indicator function of set $A$.
In this case, $L_1$ is naturally related to the {\it Hamming distance} \citep{CentroidEstimators2008}.
Then also $R_t(s_t|x^T)=1-p_t(s_t|x^T)$ so that the corresponding risk is
\begin{equation}\label{eqn:R1}
R_1(s^T|x^T):=1-{1\over T}\sum_{t=1}^T p_t(s_t|x^T).
\end{equation}
Let $v(\cdot; 1)$ stand for the Bayes classifier relative to the $R_1$-risk. 
It is easy to see from the above definition of $R_1$, that $v(\cdot; 1)$
delivers PMAP paths, which  minimize the expected number of misclassification errors. 
In addition to maximizing $\sum_{t=1}^T
p_t(s_t|x^T)$, $v(\cdot; 1)$ also maximizes 
$\prod_{t=1}^Tp_t(s_t|x^T)$, and therefore minimizes the following risk
\begin{eqnarray}\label{eqn:barr1}
\bar{R}_1(s^T|x^T):=-{1\over T}\sum_{t=1}^T\log p_t(s_t|x^T).
\end{eqnarray}
\subsection{Generalizations}\label{sec:generalizations}
\subsubsection{Admissible PMAP and Posterior Viterbi Decoders}\label{sec:pvdmore}
Recall (subsection \ref{sec:pathdiffers}) that PMAP paths can be 
{\it inadmissible}.  According to our definition of admissibility (subsection \ref{sec:pathdiffers}), 
a path is inadmissible if it is of zero posterior probability.  Although no explicit definition of
admissibility, or {\it validity}, is  given in
\citep{tutorial}, as an example of how a path can be ``not valid'' Rabiner refers to forbidden transitions,
i.e. zero prior probability (which, of course, also implies zero posterior probability); 
the possibility of a path to have a positive
prior probability but zero posterior probability is not discussed there.  As far as we are aware,
the first publication to formally write down an amended PMAP optimization problem to guarantee path validity, or
admissibility, is \citep{KallKrogh2005}. However, they too do not state explicitly 
if ``a possible path through the model'' means for them  positivity only of the prior probability or 
also of the posterior probability. If ``the model'' is to be understood
as the HMM in its entirety, then this would require positivity of the posterior probability. However, 
the optimization presented in \citep{KallKrogh2005} does not guarantee
positivity of the posterior probability, i.e. it only guarantees positivity of the prior probability.   
Perhaps, it does not happen very often in practice that the PMAP decoder constrained to return 
{\it a priori} possible paths  returns an inadmissible path (it does not happen
in our own experiments in Section  \ref{sec:example} as all of our emission 
probabilities are non-zero on the entire emission alphabet).  However, as the example 
in Appendix \ref{sec:appendixposprior0posterior} shows, this is indeed possible. 

Thus, to  enforce admissibility properly, 
$R_1$-risk needs to be minimized over the admissible paths 
($R_1$ minimization over the paths of positive prior probability
is revisited in subsection \ref{sec:beyondPVD} below):
\begin{equation}\label{eq:remove0s}
\min_{s^T: p(s^T|x^T)>0}R_1(s^T|x^T)\quad \Leftrightarrow \quad
\max_{s^T: p(s^T|x^T)>0}\sum_{t=1}^Tp_t(s_t|x^T).
\end{equation}
Assuming that $p_t(s|x^T)$, $1\le t\le T$, $s\in S$, have been
precomputed (e.g. by the classical forward-backward recursion \citep{tutorial}), a
solution to \eqref{eq:remove0s}  can be easily found by a
Viterbi-like recursion \eqref{rec1a}

\begin{eqnarray}\label{rec1a}
\delta_{1}(j)&:=&p_1(j|x^T),~~\forall~j\in S,\\
\nonumber \delta_{t+1}(j)&:=&\max_i \left(\delta_t(i) + \log
  r_t(i,j)\right)+p_{t+1}(j|x^T)
~\text{for}~t=1,2,\ldots,T-1,~\text{and}~\forall j\in S,
\end{eqnarray}
where 
$r_t(i,j):=\mathbb{I}_{\{p_{ij}f_j(x_{t+1})>0\}}$. {\it To the best of our
knowledge this has not been stated  in the literature before.} 
We will refer to this decoder as the {\it Constrained PMAP} decoder. 

Next note that in the presence of path constraints, minimization of the $R_1$-risk \eqref{eqn:R1}
is {\it no longer equivalent} to minimization of the $\R_1$-risk \eqref{eqn:barr1}. In particular, the
 problem \eqref{eq:remove0s} is not equivalent to the following 
 problem 
\begin{equation}\label{eq:remove0slog}
\min_{s^T: p(s^T|x^T)>0}\R_1(s^T|x^T)\quad \Leftrightarrow \quad
\max_{s^T: p(s^T|x^T)>0}\sum_{t=1}^T \log p_t(s_t|x^T).
\end{equation}
It is also important to note that the problem \eqref{eq:remove0slog} above {\it is equivalent}
to what has been termed the {\it posterior-Viterbi decoding}, or {\it PVD} \citep{PMAP2005}:
\begin{equation*}
\min_{s^T: p(s^T)>0}\R_1(s^T|x^T)\quad \Leftrightarrow \quad
\max_{s^T: p(s^T)>0}\sum_{t=1}^T \log p_t(s_t|x^T),
\end{equation*}
i.e. unlike in the case of $R_1(s^T|x^T)$ minimization, minimization of $\R_1(s^T|x^T)$
over the paths of positive prior probability is indeed sufficient to produce 
admissible paths. 

A solution to \eqref{eq:remove0slog} can  be computed by a related
recursion given in 
\eqref{rec3} below
\begin{eqnarray}\label{rec3}
\delta_{1}(j)&:=&\log p_1(j|x^T),~\forall j\in S,\\
\nonumber \delta_{t+1}(j)&:=&\max_i \big(\delta_t(i)+ \log
r_{ij}\big)+\log p_{t+1}(j|x^T),
~\text{for}~t=1,2,\ldots,T-1,~\forall j\in S,\end{eqnarray}
where 
$r_{ij}:=\mathbb{I}_{\{p_{ij}>0\}}$ (which for inhomogeneous chains will depend on $t$). 

\subsubsection{Beyond  PVD and {\it a priori} admissible PMAP}\label{sec:beyondPVD}
Although admissible minimizers  of $R_1$ and $\R_1$ risk are by
definition of positive probability, this probability can still be
very small. Indeed, in the above recursions, the weight
$r_{ij}$ is 1 even when $p_{ij}$ is very small. We next replace
$r_{ij}$  by the true transition probability
$p_{ij}$  in minimizing the $\R_1$-risk (i.e. maximization of
$\prod_{t=1}^Tp_t(s_t|x^T)$). Then the solutions remain admissible and
also tend to maximize the prior path probability. To bring the newly 
obtained optimization problem to a more elegant form \eqref{eq:rinftyprior0}, 
we pretend that $\delta_1(j)$ in \eqref{rec3} above was defined as 
$\delta_1(j):=\log p_1(j|x^T)+\log \mathbb{I}_{\{\pi_j>0\}}$ (which indeed does not change the results
of the recursion \eqref{rec3}) and replace the last term by $\log \pi_j$.  

Thus, with the above
replacements, 
the recursion \eqref{rec3} now solves the
following {\it seemingly unconstrained} optimization problem (see Theorem \ref{dyn})
\begin{equation}\label{eq:rinftyprior0}
\max_{s^T}\Big[\sum_{t=1}^T\log p_t(s_t|x^T)+\log p(s^T)\Big]\quad
\Leftrightarrow \quad \min_{s^T}\Big[\bar{R}_1(s^T|x^T)+h(s^T)\Big],
\end{equation}
where the penalty term
\begin{eqnarray}
  \label{eq:rinftyprior}
h(s^T)&=-{1\over T}\log p(s^T)&=:\bar R_{\infty}(s^T)
\end{eqnarray}
is the logarithmic risk based on the prior distribution~\footnote{More 
generally, the same type of risk (e.g. $\bar R_{\infty}$) can be based on the posterior ($p(s^T|x^T)$), 
joint ($p(s^T, x^T)$) or  prior ($p(s^T)$) distribution.  Compromising between notational accuracy on the one hand
 and notational simplicity and consistency on the other hand, 
throughout the paper we disambiguate these cases  solely by the argument.}, which does not involve the observed data.

The thereby modified recursions immediately generalize as
follows:
\begin{eqnarray}\label{rec4}
\delta_1(j)&:=&\log p_1(j|x^T)+C\log \pi_j,~\forall j\in S,\\
\nonumber \delta_{t+1}(j)&:=&\max_i \big(\delta_t(i)+
 C\log p_{ij}\big)+\log p_{t+1}(j|x^T)~\text{for}~t=1,2,\ldots,T-1,~\forall j\in S,
\end{eqnarray}
solving
\begin{equation}\label{gen-log-consensus}
\min_{s^T}\Big[\bar{R}_1(s^T|x^T)+Ch(s^T )\Big],
\end{equation}
where $C>0$ is a trade-off constant, which can also be
viewed as a regularization parameter.  
Indeed, Proposition \ref{prop:admit} below states that $C>0$ implies
admissibility of solutions to \eqref{gen-log-consensus}. In particular, PVD, i.e. the problem 
solved by the original 
recursion  \eqref{rec3}, can now be 
recovered by taking $C$ sufficiently small.
(Alternatively, the PVD problem can also be formally written in the
form \eqref{gen-log-consensus} with $C=\infty$ and $h(s^T )$ given,
for example, by $\mathbb{I}_{\{p(s^T)=0\}}$.)

What if the actual probabilities $p_{ij}$ ($\pi_j$) were also used in
the optimal accuracy/PMAP decoding? 
To motivate this, we re-consider the optimal accuracy/PMAP decoding imposing 
the positivity constraint not on the posterior but
on the prior path probability:
\begin{equation}\label{eq:remove0s2}
\min_{s^T: p(s^T)>0}R_1(s^T|x^T)\quad \Leftrightarrow \quad
\max_{s^T: p(s^T)>0}\sum_{t=1}^Tp_t(s_t|x^T).
\end{equation}
Solution to \eqref{eq:remove0s2}  can be easily found by yet another
Viterbi-like recursion given in  \eqref{rec1b} below
\begin{eqnarray}\label{rec1b}
\delta_{1}(j)&:=&p_1(j|x^T),~~\forall~j\in S,\\
\nonumber \delta_{t+1}(j)&:=&\max_i \left(\delta_t(i) + \log
  r_{ij}\right)+p_{t+1}(j|x^T)
~\text{for}~t=1,2,\ldots,T-1,~\text{and}~\forall j\in S.
\end{eqnarray}


We again replace the indicators $r_{ij}$ by the actual probabilities $p_{ij}$.
We  once more pretend that $\delta_1(j)$ in \eqref{rec1b} above was defined, this time, as 
$\delta_1(j):=p_1(j|x^T)+\log \mathbb{I}_{\{\pi_j>0\}}$. Replacing the last term
by  $\log \pi_j$ yields the following problem:
\begin{equation}\label{pre-gen-L1-pen}
\max_{s^T}\Big[\sum_{t=1}^T p_t(s_t|x_t)+\log p(s^T)\Big]\quad \Leftrightarrow
\quad \min_{s^T}\Big[R_1(s^T|x^T) + \R_{\infty}(s^T)\Big].
\end{equation}
A more general problem can be written in the form
\begin{equation}\label{gen-L1-pen}
 \min_{s^T}\Big[R_1(s^T|x^T)+Ch(s^T)\Big],
\end{equation}
where $h$ is some penalty function (independent of the data $x^T$). 
Thus, the problem \eqref{eq:remove0s2} of optimal accuracy/PMAP decoding   over the 
paths of positive prior probability
 is obtained by taking $C$ sufficiently small
and $h(s^T)=\R_{\infty}(s^T)$. (Setting $C\times h(s^T )=\infty\times \mathbb{I}_{\{p(s^T)=0\}}$ 
also reduces the  problem \eqref{gen-L1-pen}  back to \eqref{eq:remove0s}.)

Clearly, if instead of \eqref{eq:remove0s2} we 
started off with \eqref{eq:remove0s} ($R_1(s^T|x^T)$ minimization over the admissible
paths), we would arrive at  $\R_{\infty}(s^T|x^T)$ in place of $\R_{\infty}(s^T)$ in 
\eqref{pre-gen-L1-pen} above.  Inclusion of $\R_{\infty}(s^T|x^T)$ more generally 
is treated next in Section \ref{sec:combrisk}.

\section{Combined risks}\label{sec:combrisk}
Motivated by the previous section, we consider the following 
general problem
\begin{equation}\label{gen-problem}
\min_{s^T}\Big[C_1\bar{R}_1(s^T|x^T)+C_2\bar{R}_{\infty}(s^T|x^T)+C_3\bar{R}_1(s^T)+C_4\bar{R}_{\infty}(s^T)\Big],
\end{equation}
where $C_i\ge 0$, $i=1,2,3,4$, $\sum_{i=1}^4 C_i>0$\footnote{For uniqueness of 
representation, one may want to additionally require $\sum_{i=1}^4 C_i=1$.}. 
This is also equivalent to
\begin{equation}\label{gen-problem1}
\min_{s^T}\Big[C_1\bar{R}_1(s^T|x^T)+C_2\bar{R}_{\infty}(s^T, x^T)+C_3\bar{R}_1(s^T)+C_4\bar{R}_{\infty}(s^T)\Big],
\end{equation} 
\begin{align}\textrm{where, recalling}~\eqref{eqn:barr1},\quad
\bar{R}_1(s^T|x^T)&=-{1\over T}\sum_{t=1}^T\log p_t(s_t|x^T),
\nonumber\\
\R_{\infty}(s^T, x^T)&:=-{1\over T}\log p(x^T,s^T) \nonumber \\
            &=-{1\over T}[\log p(s^T)+\sum_{t=1}^T \log
f_{s_t}(x_t)]\nonumber\\
&=-{1\over T}[\log \pi_{s_1}+ \sum_{t=1}^{T-1}\log
p_{s_t s_{t+1}}+\sum_{t=1}^T \log f_{s_t}(x_t)],\nonumber \\
\quad\text{recalling}~\eqref{loglikerisk},\quad\bar{R}_{\infty}(s^T|x^T)&=-{1\over T}\log p(s^T|x^T),\nonumber \\
                         &=\R_{\infty}(s^T,x^T)+{1\over T}\log p(x^T),\nonumber \\
\bar{R}_1(s^T)&:=-{1\over T}\sum_{t=1}^T\log p_t(s_t),\label{eqn:barR1prior}\\
\bar{R}_{\infty}(s^T)&=-{1\over T}\log p(s^T),\quad\text{recalling}~\eqref{eq:rinftyprior},\nonumber \\
                     &=-{1\over T}[\log \pi_{s_1}+
\sum_{t=1}^{T-1}\log p_{s_t s_{t+1}}].
\end{align}
The newly introduced risk $\bar{R}_{1}(s^T)$  involves only 
the prior marginals. Note that the combination $C_1=C_3=C_4=0$
corresponds to the MAP/Viterbi decoding; the combination $C_2=C_3=C_4=0$ yields
the PMAP case, whereas the combinations $C_1=C_2=C_3=0$ and
$C_1=C_2=C_4=0$ give the {\em maximum a priori} decoding
and {\em marginal prior mode} decoding, respectively. The case
$C_2=C_3=0$ subsumes \eqref{gen-log-consensus} and the case
$C_1=C_3=0$ is the problem
\begin{equation}\label{gen-viterbi}\min_{s^T}\Big[\bar{R}_{\infty}(s^T|x^T)+C\bar{R}_{\infty}(s^T)\Big].
\end{equation}
Thus, a solution to \eqref{gen-viterbi} is a generalization of the
Viterbi decoding that allows one to suppress ($C>0$) contribution of
the data. 
\begin{remark}\label{rem:admissible}If $C_2>0$, then every solution
of \eqref{gen-problem} is admissible and the minimized risk is finite.
 \end{remark}
{\it No less important and perhaps a little less obvious is that $C_1, C_4>0$ also guarantees admissibility of the solutions}, as stated in 
Proposition \ref{prop:admit} below.
\begin{proposition}\label{prop:admit}Let
$C_1, C_4>0$. Then, the minimized risk \eqref{gen-problem} is finite and
any minimizer $s^T$ is admissible. 
\end{proposition}
\begin{proof}Without loss of generality, assume $C_2=C_3=0$. 
Since $p(x^T)>0$ (assumed in the beginning of Section \ref{sec:risk}), there exists
some admissible path $s^T$. Clearly, the combined risk of this path is finite, hence
so is the minimum risk.  Now, suppose $s^T$ is a minimizer of the combined risk and
suppose further that $s^T$ is inadmissible, i.e. $p(s^T|x^T)=0$. 
Since the minimized risk \eqref{gen-problem} is finite, we must 
have $p(s^T)>0$. Therefore, it must be that 
$p(x^T|s^T)=0$, and therefore we must have some $t$, $1\le t\le T$, such $f_{s_t}(x_t)=0$.
This would imply that any path through $(t,s_t)$ is inadmissible, hence $p_t(s_t|x^T)$, 
the sum of the posterior probabilities of all such paths, is zero. This implies  $\bar R_1(s^T|x^T)=\infty$,
contradicting optimality of $s^T$.
\end{proof}
\begin{remark}\label{rem:pvd}
Note that for any $x^T$, the Posterior-Viterbi decoding \citep{PMAP2005} (Problem \eqref{eq:remove0slog}
above) can be obtained by 
setting $C_3=C_4=0$ and taking $0<C_2\ll C_1$. Also, PVD can be obtained {\it almost surely} 
by setting $C_2=C_3=0$ and taking $0<C_4\ll C_1$.
\end{remark} It is fairly intuitive that PVD can be realized as a solution to \eqref{gen-problem}
for some $0<C_1, C_2$, when $C_3=C_4=0$.  Nonetheless, let us prove this formally.
\begin{proof}Assume $C_3=C_4=0$. For each $C_1, C_2>0$, let  $\hat y^T_{C_1,C_2}\in S^T$ 
be a solution to \eqref{gen-problem}. Thus, we have
\begin{eqnarray*}
C_1\bar{R}_1(\hat y^T_{C_1,C_2}|x^T)+C_2\bar{R}_{\infty}(\hat y^T_{C_1,C_2}|x^T)&\le &
C_1\bar{R}_1(\hat y_{PVD}^T|x^T)+C_2\bar{R}_{\infty}(\hat y_{PVD}^T|x^T).
\end{eqnarray*}
 Then
$$0\le C_1(\bar{R}_1(\hat y^T_{C_1,C_2}|x^T)-\bar{R}_1(\hat y_{PVD}^T|x^T))\le 
C_2(\bar{R}_{\infty}(\hat y_{PVD}^T|x^T)-\bar{R}_{\infty}(\hat y^T_{C_1,C_2}|x^T)) $$
holds for any $C_1,C_2>0$. 
Since $\bar{R}_{\infty}(\hat y_{PVD}^T|x^T)-\bar{R}_{\infty}(\hat y^T_{C_1,C_2}|x^T)$ is 
clearly bounded (and $S^T$ is finite), by allowing $C_2$ to be arbitrarily small, we obtain 
$\bar{R}_1(\hat y^T_{C_1,C_2}|x^T)=\bar{R}_1(\hat y_{PVD}^T|x^T)$ for some sufficiently small $C_2$.
Since $C_2>0$,  all $\hat y^T_{C_1,C_2}$ are 
admissible (Remark \ref{rem:admissible} above), therefore for such sufficiently small $C_2$, 
$\hat y^T_{C_1,C_2}$ is also a solution to the PVD Problem \eqref{eq:remove0slog}.

The second statement is proved similarly, recalling Proposition \eqref{prop:admit} 
to establish admissibility of $\hat y^T_{C_1,C_4}$ {\it almost surely}.
\end{proof}

If the smoothing probabilities $p_t(s|x^T)$,
$t=1,\ldots, T$ and $s\in S$, have been already computed, 
a solution to \eqref{gen-problem} can be found also by a standard dynamic
programming algorithm. Let us first introduce more notation. For
every $t \in 1,\ldots,T$ and $j\in S$, let
$$\gamma_t(j):=C_1\log p_{t}(j|x^T)+C_2\log f_{j}(x_{t})+C_3\log
p_t(j).$$ Note that the function $\gamma_t$ depends on the entire data
$x^T$. Next, let us also define the following scores
\begin{eqnarray}\label{gen-rec}
\delta_1(j)&:=&(C_2+C_4)\log \pi_j+\gamma_1(j), ~\forall j\in S, \nonumber \\ 
\label{ind} \delta_{t}(j)&:=&\max_i
\big(\delta_{t-1}(i)+ (C_2+C_4)\log p_{ij}\big)+\gamma_{t}(j)\\ \nonumber
&&  ~\text{for}~t=2,3\ldots,T,~\text{and}~\forall j\in S.
\end{eqnarray}
Using the above scores $\delta_t(j)$ and a suitable tie-breaking rule, 
below we define the back-pointers $i_t(j)$, terminal state $i_T$,  and the optimal 
path $\hat s^{T}(i_T)$.
\begin{align}\label{arg}
i_t(j)&:=  
           \arg\max_{i\in S}[\delta_t(i)+(C_2+C_4)\log p_{ij}], \quad \hbox{when $t=1,\ldots, T-1$;} \nonumber  \\
 i_T&:=          \arg\max_{i\in S}\delta_T(i).\\
\label{st}
\hat s^{t}(j)&:=\left\{
                \begin{array}{ll}
                  i_1(j), & \hbox{when $t=1$;} \\
                  \big(\hat s^{t-1}(i_{t-1}(j)),j \big) & \hbox{when $t=2,\ldots,T$.}
                \end{array}
              \right.
\end{align}
The following theorem formalizes the dynamic programming argument; 
its proof is standard and we state it below for completeness only.
\begin{theorem}\label{dyn}
Any solution to \eqref{gen-problem} can be represented in the form $\hat s^T(i_T)$
provided the ties in \eqref{arg} are broken accordingly.
\end{theorem}
\begin{proof}
With a slight abuse of notation,  for every $s^t\in S^t$, let 
$$U(s^t)=\sum_{u=1}^t\left[ \gamma_u(s_u)+(C_2+C_4)\log
p_{s_{u-1}s_u}\right],$$ where $s_0:=0$ and $p_{0s}:=\pi_s$. Hence,
$$-T[C_1\bar{R}_1(s^T|x^T)+C_2\bar{R}_{\infty}(s^T, x^T)+C_3\bar{R}_1(s^T)+C_4\bar{R}_{\infty}(s^T)]=U(s^T)$$
and any maximizer  of $U(s^T)$ is clearly a solution to \eqref{gen-problem} and \eqref{gen-problem1}.

Next, let $U(j):=\delta_1(j)$ for all $j\in S$, and let
$$U(s^{t+1})=U(s^t)+(C_2+C_4)\log p_{s_t s_{t+1}}+\gamma_{t+1}(s_{t+1}),$$
for $t=1,2,\ldots,T-1$ and also $s^t\in S^t$. By induction on $t$, these
yield
$$\delta_t(j)=\max_{s^t: s_t=j}U(s^t)$$ 
for every $t=1,2,\ldots,T$ and for all $j\in S$.
Clearly, every maximizer $\hat s^T$ of $U(s^T)$ over the set $S^T$
must satisfy  $\hat s_T=i_T$, or, more precisely $\hat s_T \in \arg\max_{j\in S}\delta_T(j)$,
allowing for  non-uniqueness. Continuing to interpret  $\arg\max$ as a set, 
recursion \eqref{ind} implies 
recursions \eqref{arg} and \eqref{st}, hence 
any maximizer $\hat s^T$ can indeed be computed in the form $\hat s^T(\hat s_T)$ via
the
{\it forward} (recursion \eqref{arg})-{\it backward} (recursion \eqref{st})
procedure.
\end{proof}
\vskip 1\baselineskip\noindent
Similarly to the generalized risk minimization of
\eqref{gen-problem}, the generalized problem of accuracy optimization  \eqref{gen-L1-pen}
can also be further generalized as follows:
\begin{equation}\label{L1-gen-problem}
\min_{s^T}\Big[C_1{R}_1(s^T|x^T)+C_2\bar{R}_{\infty}(
s^T|x^T)+C_3{R}_1(s^T)+ C_4\bar{R}_{\infty}(s^T)\Big],
\end{equation}
where risk
\begin{equation}\label{eqn:R1prior}
{R}_1(s^T):=\frac{1}{T}\sum_{t=1}^T\P(Y_t\ne s_t)=1-\frac{1}{T}\sum_{t=1}^Tp_t(s_t)
\end{equation}
is the error rate relative to the prior distribution. 
This problem  can be solved by the following recursion
\begin{eqnarray}\label{L1-gen-rec}
\delta_1(j)&:=&(C_2+ C_4)\log \pi_j+\gamma_1(j), ~\forall j\in S, \nonumber\\ 
\delta_{t}(j)&:=&\max_i \big(\delta_{t-1}(i)+ (C_2+C_4)\log
p_{ij}\big)+\gamma_{t}(j),\\ \nonumber
&&  ~\text{for}~t=2,3\ldots,T,~\text{and}~\forall j\in S,
\end{eqnarray}
where now
$$\gamma_t(j)=C_1p_t(j|x^T)+C_2\log f_j(x_{t})+C_3p_t(j).$$
The following remarks compare this generalized Problem with the generalized Problem 
\eqref{gen-problem} (Remarks \ref{rem:admissible} and \ref{rem:pvd}, Proposition \ref{prop:admit}).
\begin{remark}\label{rem:ConstrPMAP}
\begin{enumerate}
 \item As in the generalized {\it posterior-Viterbi decoding} \eqref{gen-problem}, here $C_2>0$ also
implies admissibility of the optimal paths.
\item Now, $C_4>0$ implies that the minimized risk is
finite for any $x^T$, but unlike in \eqref{gen-problem},
$C_1, C_4>0$ is not sufficient to guarantee  admissibility {\it almost surely} of the solutions to the
problem \eqref{L1-gen-problem}.  
\item Taking $C_3=C_4=0$, the constrained PMAP problem  \citep{KallKrogh2005} (Problem \eqref{eq:remove0s} above) 
is obtained for some $C_1$, $C_2$ such that $0<C_2\ll C_1$.
\end{enumerate}
\end{remark}
We refer to a decoder  solving the generalized risk minimization Problem \eqref{gen-problem}
as a {\it generalized posterior-Viterbi hybrid decoder}.
 Similarly, a decoder solving  the generalized optimal accuracy 
Problem \eqref{L1-gen-problem} is referred to as a {\it generalized PMAP 
hybrid decoder} to distinguish  the product-based risk $\R_1(s^T|x^T)$ 
in the former case from the sum-based risk $R_1(s^T|x^T)$ in the latter case. 
Both the generalized families, however, naturally extend the 
PMAP/optimal accuracy/posterior decoder (Section \ref{subsec:viterbipmap}).  

To further characterize  the solutions to these generalized problems, we next state a simple
general result.  
\begin{lemma}\label{lemma:optimization}
Let $F$ and $G$ be functions from a set $A$ to the extended reals $\mathbb{\bar R}=\mathbb{R}\cup\{\pm \infty\}$.
Let $\alpha_1,\alpha_2\in [0,1]$ be such that $\alpha_1\le \alpha_2$. Suppose $a_1, a_2 \in A$ are such that
\begin{eqnarray*}
\alpha_i F(a_i)+(1-\alpha_i)G(a_i)&\le&  \alpha_i F(x)+(1-\alpha_i)G(x)\quad i=1,2~\text{for all}~x\in A.
\end{eqnarray*}
Then $F(a_1)\ge F(a_2)$ and $G(a_1)\le G(a_2)$.
\end{lemma} Although the result is obvious, below we state its proof for completeness.
\begin{proof} Write $a$, $b$, $c$, and $d$ for $F(a_1)$, $G(a_1)$, $F(a_2)$, and $G(a_2)$, respectively.
 Then we have 
\begin{eqnarray*}
\alpha_1(a-c)\le (1-\alpha_1)(d-b)\\
\alpha_2(a-c)\ge (1-\alpha_2)(d-b)
\end{eqnarray*}
and therefore
\begin{eqnarray*}
\alpha_2\alpha_1(a-c)\le \alpha_2(1-\alpha_1)(d-b)\\
\alpha_1\alpha_2(a-c)\ge \alpha_1(1-\alpha_2)(d-b),
\end{eqnarray*}
which gives $\alpha_1(1-\alpha_2)(d-b)\le \alpha_2(1-\alpha_1)(d-b)$.
Since $\alpha_1(1-\alpha_2)\le \alpha_2(1-\alpha_1)$, it follows that $d\ge b$, i.e. $G(a_2)\ge G(a_1)$.
The fact that $F(a_1)\ge F(a_2)$ is obtained similarly.
\end{proof}
\begin{corollary}\label{corollary:optimization}
\begin{enumerate}
 \item Let $\hat y$ and $\hat y'$ be solutions to Problem \eqref{gen-problem} with $C_1\in [0,1]$ and $C_2=1-C_1$, $C_3=C_4=0$
and $C'_1\in [0,1]$ and $C'_2=1-C'_1$, $C'_3=C'_4=0$, respectively.  Assume $C_1\le C_1'$.
Then $\bar R_1(\hat y|x^T) \ge \bar R_1(\hat y'|x^T)$ and 
$\bar R_\infty(\hat y|x^T) \le \bar R_\infty(\hat y'|x^T)$.
\item Let $\hat y$ and $\hat y'$ be solutions to Problem \eqref{gen-problem} with $C_3\in [0,1]$ and $C_4=1-C_3$, $C_1=C_2=0$
and $C'_3\in [0,1]$ and $C'_4=1-C'_3$, $C'_1=C'_2=0$, respectively. Assume $C_3\le C_3'$. Then $\bar R_1(\hat y) \ge \bar R_1(\hat y')$ and 
$\bar R_\infty(\hat y) \le \bar R_\infty(\hat y')$.
\item Let $\hat y$ and $\hat y'$ be solutions to Problem \eqref{L1-gen-problem}
 with $C_1\in [0,1]$ and $C_2=1-C_1$, $C_3=C_4=0$ and $C'_1\in [0,1]$ and $C'_2=1-C'_1$, $C'_3=C'_4=0$, respectively.
Assume $C_1\le C_1'$. Then $R_1(\hat y|x^T) \ge R_1(\hat y'|x^T)$ and 
$\bar R_\infty(\hat y|x^T) \le \bar R_\infty(\hat y'|x^T)$.
\item Let $\hat y$ and $\hat y'$ be solutions to Problem \eqref{L1-gen-problem} with $C_3\in [0,1]$ and $C_4=1-C_3$, $C_1=C_2=0$
and $C'_3\in [0,1]$ and $C'_4=1-C'_3$, $C'_1=C'_2=0$. Assume $C_3\le C_3'$. Then $R_1(\hat y) \ge R_1(\hat y')$ and 
$\bar R_\infty(\hat y) \le \bar R_\infty(\hat y')$.
\end{enumerate}
\end{corollary}
\begin{proof}
 A straightforward application of Lemma \ref{lemma:optimization}.
\end{proof}

\section{The $k$-block Posterior-Viterbi decoding}\label{sec:bridge}
The next approach 
provides a 
surprisingly different insight into what otherwise has already been formulated as the generalized Problem \eqref{gen-problem}.
This, first of all, helps better understand how the generalized Problem \eqref{gen-problem} resolves
the drawback of Rabiner's suggestion (introduced in the last paragraph of Subsection \ref{sec:pathdiffers} above).  
Secondly, 
the same approach gives an elegant relationship (Theorem \ref{k-block}, Corollary \ref{corollary0}) between 
the main types of risk,
which surprisingly amounts to, as far as we know, a novel property of ordinary Markov chains 
(equation \eqref{k-blockrec2},  and Proposition
\ref{prop:generalMC} of the concluding Section \ref{sec:discussion}).  

Recall (subsection \ref{sec:segment}) that Rabiner's compromise between MAP and PMAP is
to maximize the expected number of correctly
decoded pairs or triples of (adjacent) states. 
With $k$ being the length of the overlapping
block ($k=2,3,\ldots$) this means to minimize the conditional risk
\begin{equation}\label{Rabiner-k}
R_k(s^T|x^T):=1-\frac{1}{T-k+1}\sum_{t=1}^{T-k+1}p(s_t^{t+k-1}|x^T)\end{equation}
which derives from the following loss function:
\begin{equation}\label{eq:dnew}
L_k(y^T,s^T):=\frac{1}{T-k+1}\sum_{t=1}^{T-k+1} \mathbb{I}_{\{ s_{t}^{t+k-1}\neq y_{t}^{t+k-1}\}}.
\end{equation}
Obviously, for $k=1$ this gives the usual $R_1$ maximization, i.e. the PMAP
decoding, which is known to fault by allowing inadmissible paths.
It is natural to think that minimizers of  $R_k(s^T|x^T)$  ``move''
towards Viterbi paths ``monotonically'' as $k$  increases to $T$.
Indeed, when $k=T$, minimization of $R_k(s^T|x^T)$ \eqref{Rabiner-k}
is equivalent to minimization of $\bar R_\infty(s^T|x^T)$ achieved
by  the Viterbi decoding. However, as the experiments in Section \ref{sec:example}
below show, minimizers of \eqref{Rabiner-k} are not guaranteed to 
be admissible (even if admissibility were defined relative to the prior distribution) 
for $k>1$. Also, as we already pointed out in subsection \ref{sec:pathdiffers},
this approach does not give monotonicity, i.e. allows the optimal path for $k=2$
to have lower (prior and posterior) probabilities than those of the PMAP path (i.e. $k=1$).
Another drawback of using the loss $L_k$ \eqref{eq:dnew} is that,
 unlike the generalized PVD and PMAP hybrid decoders, 
the computational complexity of the Rabiner approach grows with the block length $k$.
We now show how these drawbacks go away when the sum in
\eqref{Rabiner-k} is replaced by a product, eventually arriving
at a subfamily of the generalized posterior Viterbi decoders. Certainly, replacing
the sum by the product alters the problem, and it does so in a way that makes
the block-wise coding idea work well. Namely, the longer the
block, the larger the resulting path probability, which is also now guaranteed to be
positive already for $k=2$. Moreover, this gives another interpretation
of the risks $\R_1(s^T|x^T)+C\R_{\infty}(s^T|x^T)$ (see also Remark \ref{rem:pvd} above), 
the prior risks $\R_1(s^T)+C\R_{\infty}(s^T)$, and consequently the generalized Problem \eqref{gen-problem}. 

Let $k$ be a positive integer. For the time being, let $p$ represent any first order Markov chain on $S^T$, 
and let us define
$$\U_k(s^T):=\prod_{j=1-k}^{T-1}p\big(s_{\max(j+1,1)}^{\min(j+k, T)}\big),\quad \R_k(s^T):=-{1\over T}\ln \U_k(s^T).$$ 

Thus
$$\U_k(s^T)=U^k_1\cdot U^k_2\cdot U^k_3,$$
where
\begin{align*}
U^k_1&:=p(s_1)\cdots p(s_1^{k-2})p(s_1^{k-1})\\
U^k_2&:=p(s_1^k)p(s_2^{k+1})\cdots
p(s_{T-k}^{T-1})p(s_{T-k+1}^T)\\
U^k_3&:=p(s_{T-k+2}^T)p(s_{T-k+3}^T)\cdots p(s_T).
\end{align*}
Thus, $\R_k$ is a natural generalization of $\R_1$ (introduced first for the posterior
distribution in \eqref{eqn:barr1}) since when $k=1$,  $\R_k=\R_1$.
\begin{theorem}\label{k-block}
Let $k$ be such that $T\geq k>1$. Then the following recursion holds
\begin{equation*}
\R_k(s^T)=\R_{\infty}(s^T)+\R_{k-1}(s^T),\quad \forall
s^T\in S^T.\end{equation*}\end{theorem}
\begin{proof}
Note that 
$$U^k_1=U^{k-1}_1p(s_1^{k-1}),\quad
U^k_3=p(s_{T-k+2}^T)U^{k-1}_3.$$
Next, for all $j$ such that $j+k\le T$, the Markov property gives 
$$p(s_{j+1}^{j+k})=p(s_{j+k}|s_{j+k-1})p(s^{j+k-1}_{j+1})$$
and
\begin{align*}
&U^k_2p(s_{T-k+2}^T)=p(s_1^k)p(s_2^{k+1})\cdots p(s_{T-k+1}^T)p(s_{T-k+2}^T)=\\
&p(s_k|s_{k-1})p(s_1^{k-1})p(s_{k+1}|s_k)p(s_2^{k})\cdots
p(s_T|s_{T-1})p(s_{T-k+1}^{T-1})p(s_{T-k+2}^T)=\\
&p(s_{k}|s_{k-1})p(s_{k+1}|s_k)\cdots p(s_T|s_{T-1})
p(s_1^{k-1})\cdots
p(s_{T-k+1}^{T-1})p(s_{T-k+2}^T)=\\
&p(s_{k}|s_{k-1})\cdots p(s_T|s_{T-1})U_2^{k-1}.\end{align*}
Hence,
\begin{align*}
\U_k(s^T)&=U^{k-1}_1p(s_1^{k-1})p(s_{k}|s_{k-1})\cdots
p(s_T|s_{T-1})U_2^{k-1}U_3^{k-1}\\
&=p(s_1^{T})U^{k-1}_1U_2^{k-1}U_3^{k-1}=p(s^T)\U_{k-1}(s^T).
\end{align*}
The second equality above also follows from the Markov property.
Taking logarithms on both sides and dividing by $-T$ completes the proof. 
\end{proof}
Now, we specialize this result to our HMM context, and, thus, $p(s^T)$ and 
$p(s^T|x^T)$ are again the prior and posterior hidden path distributions. 
\begin{corollary}\label{corollary0}Let $k$ be such that $T\geq k>1$. For all paths 
$s^T\in S^T$ the prior risks $\R_k$ and $\R_\infty$
satisfy \eqref{recursion1}. For every $x^T\in {\mathcal X}^T$
and for all paths $s^T\in S^T$, the posterior risks $\R_k$ and $\R_\infty$ 
satisfy \eqref{recursion}.
\begin{eqnarray}
\R_k(s^T)&=&\R_{\infty}(s^T)+\R_{k-1}(s^T), \label{recursion1}\\
\R_k(s^T|x^T)&=&\R_{\infty}(s^T|x^T)+\R_{k-1}(s^T|x^T). \label{recursion}
\end{eqnarray}
\end{corollary}
\begin{proof}Clearly, conditioned on the data $x^T$, $Y^T$ remains a first
order Markov chain (generally inhomogeneous even if it was   homogeneous {\it a priori}).
Hence, Theorem \ref{k-block} applies.\end{proof}

Below, {\it we focus on the posterior distribution and risks, but the discussion readily
extends to any first order Markov chain.}

Let $v(x^T; k)$ be a decoder that minimizes $\R_k(s^T|x^T)$,
\begin{equation}\label{eq:khybrids}
v(x^T;k)=\arg\max_{s^T\in S^T}\U_k(s^T|x^T)=\arg\min_{s^T\in S^T}\R_k(s^T|x^T).
\end{equation} 
Corollary \eqref{corollary} below states how $\R_k(s^T|x^T)$ minimization is a special 
case of the generalized Problem \eqref{gen-problem}. 
We refer to the generalized posterior-Viterbi hybrid decoders $v(x^T;k)$ as {\it k-block} PVD
and summarize their properties in Corollary \eqref{corollary}.
\begin{corollary}\label{corollary}
For every $x^T\in {\mathcal X}^T$, and for every $s^T\in S^T$, we have
\begin{align}\label{k-blockrec2}
&\R_k(s^T|x^T)=(k-1)\R_{\infty}(s^T|x^T)+\R_{1}(s^T|x^T)&\forall k~\text{such that}~1\le k \le T.\\
&v(x^T; k)~\text{is admissible}& \forall k~\text{such that}~k>1. \label{eqn:admissible}\\ 
\label{eqn:increase}
&\R_{\infty}(v(x^T; k)|x^T)\leq \R_{\infty}(v(x^T; k-1)|x^T)& \forall k~\text{such that}~1<k \le T.\\
\label{eqn:increase2}
&\R_{1}(v(x^T; k)|x^T)\geq \R_{1}(v(x^T; k-1)|x^T)&\forall k~\text{such that}~1<k \le T.
\end{align}
\end{corollary}
Equation \eqref{k-blockrec2}
is also of practical significance showing that $v(x^T; k)$ 
is a solution to \eqref{gen-problem} with $C_1=1$,
$C_2=k-1$, $C_3=C_4=0$, and as such can be computed in the same
fashion for all $k$, $1\le k\le T$ (see Theorem \ref{dyn} above). 

Inequality \eqref{eqn:increase} means that
the posterior path probability $p(v(x^T; k)|x^T)$ increases with $k$. 
At the same time, increasing $k$ also increases $\R_1$-risk, i.e. decreases the
product of the (posterior) marginal probabilities of states along the path 
$v(x^T; k)$. 
Inequalities \eqref{eqn:increase} and \eqref{eqn:increase2} clearly show that as
$k$ increases, $v(\cdot; k)$ monotonically moves from $v(\cdot; 1)$ (PMAP) 
towards the Viterbi decoder, i.e. $v(\cdot; \infty)$.  
However, the maximum block length is $k=T$. 

A natural way to complete this bridging of PMAP with MAP is by embedding 
the $\R_k$ risks into the family $\R_\alpha$ via
$\alpha=\frac{k-1}{k}\in [0,1]$. Thus,  \eqref{k-blockrec2} extends to
\begin{equation}\label{k-blockrec22}
\R_\alpha(s^T|x^T):=\alpha\R_{\infty}(s^T|x^T)+(1-\alpha) \R_{1}(s^T|x^T)
\end{equation}
with $\alpha=0$ and $\alpha=1$ corresponding to the PMAP and Viterbi cases, respectively. 
This embedding is clearly still within the generalized Problem \eqref{gen-problem}
via $C_1=1-\alpha$, $C_2=\alpha$, $C_3=C_4=0$. In particular, $v(x^T; k(\alpha))$
can be computed by using the same dynamic programming algorithm of Theorem \ref{dyn}
for all $k\in [1,\infty]$ (i.e. all $\alpha\in[0,1]$), and inequalities \eqref{eqn:increase} 
and \eqref{eqn:increase2} are special cases of Corollary \ref{corollary:optimization} (part 1) to Lemma \ref{lemma:optimization}.

Recalling Remark \ref{rem:pvd}, we note
that on the lower end of $0\le \alpha\le 1$, before reaching PMAP ($\alpha=0$) 
we encounter PVD for some sufficiently small $\alpha\approx 0$. 
Note also that in \eqref{eqn:admissible} $k$ need not be integer either, i.e. 
Remark \ref{rem:admissible} establishes admissibility of $v(x^T; k(\alpha))$, $k(\alpha)=1/(1-\alpha)$,  
for all $\alpha \in (0,1]$ (i.e. all $k\in (1,\infty]$).  
\begin{proof}
Equation \eqref{k-blockrec2} follows immediately from equation \eqref{recursion} of 
Corollary \ref{corollary0}. Admissibility of $v(x^T; k)$ for $k>1$ in \eqref{eqn:admissible}
becomes obvious recalling Remark \ref{rem:admissible}. Inequalities \eqref{eqn:increase} and
\eqref{eqn:increase2} are established by Corollary \ref{corollary:optimization}.
 \end{proof}


Given $x^T$ and a sufficiently  large $k$ (equivalently, $\alpha\approx 1$), 
$v(x^T; k)$, the minimizer of $\R_\alpha(s^T|x^T)$ \eqref{k-blockrec22} 
(and \eqref{k-blockrec2}) would produce a Viterbi path $v(x^T; \infty)$
(since $S^T$ is finite).  However, such $\alpha$ (and $k$) would generally 
depend on $x^T$, and in particular $k$ may need to be larger than $T$, i.e.
$v(x^T; T)$ may be different from $v(x^T; \infty)$.

At the same time, for $k>1$ we  have 
\begin{equation}\label{vorratused}
\R_{\infty}(v(x^T; \infty)|x^T)\leq \R_{\infty}(v(x^T; k)|x^T)\leq
\R_{\infty}(v(x^T; \infty)|x^T) +{\R_1(v(x^T; \infty)|x^T)\over k-1},
\end{equation}
on which we comment more in Section \ref{sec:asy} below.
The first inequality of \eqref{vorratused} above follows
immediately from the definition of the Viterbi decoder. To obtain the second inequality, 
apply \eqref{k-blockrec2} to both $v(x^T; k)$ and $v(x^T; \infty)$ and
subtract one equation from the other.  Dividing the resulting terms by $k-1$, noticing that 
$\R_k(v(x^T; \infty)|x^T)\ge \R_k(v(x^T; k)|x^T)$ and 
$\R_1(v(x^T; k)|x^T)\ge 0$, and rearranging the other terms yields the result. 

Considering the prior chain $Y^T$ and risks in \eqref{recursion1}, we immediately
obtain statements analogous to \eqref{k-blockrec2}-\eqref{eqn:increase2}, 
extending these new interpretations to the entire generalized Problem \eqref{gen-problem}.

\section{Algorithmic approaches}\label{subsec:alg}
It is also possible (at least when the Viterbi path is unique) to hybridize MAP and PMAP inferences 
without introduction of risk/loss functions.
We discuss such approaches mainly because one such approach was taken by \citep{Brushe1998}, the only publication 
dedicated to the theme of hybridization of the MAP and PMAP inferences in HMMs.

First note that the hybridization can be achieved by a suitable transformation of the  forward and
 backward variables $\alpha_t(i)$ and $\beta_t(i)$ defined in \eqref{eqn:fwv}. To make this 
concrete, consider the recursively applied power transformations with $\mu>0$ given in \eqref{eq:abp} below
\begin{eqnarray}\label{eq:abp}
\alpha_1(i; \mu)& := &\alpha_1(i)   \\
\alpha_{t}(i; \mu)& :=& \left[\sum_{j=1}^K\left(\alpha_{t-1}(j; \mu)p_{ji}\right)^\mu\right]^\frac{1}{\mu}f_i(x_{t}),\quad t=2,3\ldots, T \nonumber\\
\beta_T(i; \mu)&:=&\beta_T(i)=1\nonumber\\
\beta_{t}(i; \mu)& :=& \left[\sum_{j=1}^K
 \left(p_{ij}f_j(x_{t+1})\beta_{t+1}(j; \mu)\right)^\mu \right]^\frac{1}{\mu},\quad t=T-1,T-2,\ldots, 1, \nonumber
\end{eqnarray}
for all $i\in S$.  Clearly, $\alpha_t(i; 1)=\alpha_t(i)$ and $\beta_t(i;
1)=\beta_t(i)$, for all $i\in S$ and all $t=1,2,\ldots, T$. Thus, 
$\mu=1$ leads to the PMAP decoding, i.e. at time $t$ returning 
\begin{equation}\label{eqn:pmaphybrids}
v_t=\arg\max_{i\in S}\{\alpha_t(i;1)\beta_t(i;1)\}, 
\end{equation}
provided some tie-breaking rule. 

Using induction on $t$ and continuity of the power transform, it can also
be seen that the following limits exist and are finite for all $i\in
S$ and all $t=1,2,\ldots, T$: $\lim_{\mu\to \infty}\alpha_t(i;
\mu)=:\alpha_t(i,\infty)$ and $\lim_{\mu\to \infty}\beta_t(i;
\mu)=:\beta_t(i; \infty)$, where 
\begin{align}\label{eqn:alphamuinf}
\alpha_t(i;\infty)=&\max_{ s^{t}: s_t=i}p(x^t, s^t),\quad t=1,2,\ldots, T,\\
                  =&\max_{j\in S}\left(\alpha_{t-1}(j;\infty)p_{ji}\right)f_i(x_t),\quad t=2,3,\ldots,T,\nonumber\\
\nonumber \beta_t(i; \infty)=&
\max_{s_{t+1}^T\in S^{T-t}}p(x_{t+1}^T, s_{t+1}^T|Y_t=i),\quad t=T-1,T-2,\ldots,1, \quad \text{and}~
\beta_T(i; \infty)=1,\\
\nonumber&\max_{j\in S}\left(p_{ij}f_j(x_{t+1})\beta_{t+1}(j; \infty)\right).
\end{align}
The above convergence follows from the following trivial observation, which
we nonetheless prove below for reasons to become clear later on in the context of
equation \eqref{eq:brushe}.  
\begin{proposition}\label{eqn:thermodynlim} 
Let $a_j(\mu)$, $j=1,2,\ldots, K$, be non-negative as functions of $\mu\in (0,\infty)$.
Assume that $a_j(\mu)$ converges to some (finite) limit $a_j$ as $\mu\to\infty$. Assume further that
for any $\mu$, at least some of the $a_j(\mu)$ are positive. Then we have
\begin{eqnarray*}
\lim_{\mu\to \infty} \left( \sum_{j=1}^K a_j(\mu)^\mu\right)^{\frac{1}{\mu}}=
\max_{1\le j\le K}\{a_j\}.
\end{eqnarray*}
\end{proposition}
\begin{proof}
Let $M(\mu)=\max_{1\le j\le K}\{a_j(\mu)\}$, and let $M=\max_{1\le j\le K}\{a_j\}$.
Write $\left( \sum_{j=1}^K a_j(\mu)^\mu\right)^{\frac{1}{\mu}}=
M(\mu)\left( \sum_{j=1}^K \left(\frac{a_j(\mu)}{M(\mu)}\right)^\mu\right)^{\frac{1}{\mu}}$ and 
note that as $\mu\to\infty$, $M(\mu)$ converges to $M$.  Also, we have
$$1\le \left(\sum_{j=1}^K\left(\frac{a_j(\mu)}{M(\mu)}\right)^\mu \right)^{\frac{1}{\mu}}\le K^{\frac{1}{\mu}}.$$
Since  $K^{\frac{1}{\mu}}\to 1$, 
by the Sandwich Theorem the middle term also converges to 1, yielding the proposed result. 
\end{proof}

Returning to \eqref{eqn:alphamuinf}, we note that
any Viterbi path $v(x^T; \infty)=(v_1,\ldots,v_T)$ satisfies the
following property: 
\begin{equation}\label{eqn:viterbi-pointwise}
v_t=\arg\max_{i\in S}\{\alpha_t(i;\infty)\beta_t(i;
\infty)\}. 
\end{equation}
The above property \eqref{eqn:viterbi-pointwise} has already been pointed out by
\citep{Brushe1998}. 
The main motivation of \citep{Brushe1998}, however,
seems to be the case of continuous emission distributions $P_s$,
which might explain why the authors do not consider the fact that
{\it not every path that satisfies \eqref{eqn:viterbi-pointwise} is necessarily
Viterbi, i.e. MAP.}  Thus, ignoring  potential non-uniqueness 
of the Viterbi paths, 
\citep{Brushe1998} state, based on \eqref{eqn:viterbi-pointwise}, 
that the Viterbi path can be found {\em symbol-by-symbol}.
As the following simple example shows, 
{\it when the Viterbi path is not unique, the attempt to implement
the Viterbi decoding in the symbol-by-symbol fashion (based on \eqref{eqn:viterbi-pointwise}) can produce
suboptimal (i.e. in the sense of MAP), or even inadmissible, paths}.
\begin{example}\label{example:BrusheMAP}
Let $S=\{1,2,3\}$ and let $\{A,B,C,D\}$ be the emission alphabet. Let the initial
distribution $\pi$, transition probability matrix $\mathbb{P}$, and the emission 
distributions $f_s$, $s\in S$, be defined as follows: 
\begin{eqnarray*}
\pi=\begin{pmatrix}
     0.4 \\0.54 \\ 0.06
    \end{pmatrix}
&  \mathbb{P}=\begin{pmatrix}
              0.6  &  0.4  &   0\\
    0.1   & 0.1 &   0.8\\
         0   & 0.02 &   0.98\end{pmatrix}
& 
\begin{array}{lllll}
&A   & B     & C     & D \\
f_1(\cdot) &0.3 & 0.15  & 0.25  & 0.3 \\
f_2(\cdot)&0.2 &0.3  &0.3  &0.2 \\
f_3(\cdot)&0.1(6) & 0.1(6) & 0.1(6) & 0.5
\end{array}.
\end{eqnarray*}
Suppose the sequence $x^2=(A,B)$ has been observed. The (posterior) probabilities of
all the nine paths $(i,j)$ are then summarized in the matrix $PP=(P(Y^2=(i,j)|AB))$ below:
$$PP=\begin{pmatrix}
      0.0108  &  0.0144 &        0\\
    0.0016  &  0.0032 &   0.0144\\
         0  &  0.0001 &   0.0016
\end{pmatrix},$$
hence there are two Viterbi paths in this case, namely $(1,2)$ and $(2,3)$.
Now, $\alpha_1(i;\infty)=\pi_i f_i(A)$, $i\in S$, and 
$\beta_1(i;\infty)= \max_{j\in S} P(X_2=B,Y_2=j|Y_1=i)= \max_{j\in S} f_j(B)p_{ij}$, or,
in the vector form:
$$\begin{pmatrix}\alpha_1(1;\infty) \\ \alpha_1(2;\infty)\\ \alpha_1(3;\infty) \end{pmatrix}
=\begin{pmatrix} 
  0.12\\
    0.108\\
    0.01
 \end{pmatrix},\quad \begin{pmatrix}\beta_1(1;\infty) \\ \beta_1(2;\infty)\\ \beta_1(3;\infty) \end{pmatrix}
=\begin{pmatrix} 
  0.12\\
    0.1(3)\\
    0.16(3)
 \end{pmatrix},\quad \begin{pmatrix}\alpha_1(1;\infty) \beta_1(1;\infty) \\ 
\alpha_1(2;\infty) \beta_1(2;\infty)\\ \alpha_1(3;\infty) \beta_1(3;\infty) \end{pmatrix}=
\begin{pmatrix} 0.0144 \\   0.0144\\   0.0016(3)
 \end{pmatrix},
$$
so we have $v_1=1$ or $v_1=2$.  On the other hand,
$\alpha_2(i;\infty)=\max_{j\in S} P(X^2=(A,B),Y^2=(j,i))$, and $\beta_2(i,\infty)=1$ for all $i\in S$.
Therefore, 
$$\begin{pmatrix}\alpha_2(1;\infty) \\ \alpha_2(2;\infty)\\ \alpha_2(3;\infty) \end{pmatrix}
=
\begin{pmatrix}\alpha_2(1;\infty) \beta_2(1;\infty)\\ \alpha_2(2;\infty)\beta_2(2;\infty)\\ 
\alpha_2(3;\infty)\beta_2(3;\infty) \end{pmatrix}=
\begin{pmatrix} \max\{0.0108,0.0016,0\} \\ \max\{0.0144,0.0032,0.0001\}\\ \max\{0, 0.0144, 0.0016\} \end{pmatrix}
=\begin{pmatrix} 0.0108 \\ 0.0144 \\ 0.0144 \end{pmatrix}.$$  Therefore, $v_2=2$ or $v_2=3$.  
However, the symbol-by-symbol decoding is not aware that gluing $v_1=1$ and $v_2=3$ is not only 
suboptimal, but is actually forbidden, i.e. results in the inadmissible path $(1,3)$.
\end{example}  
 
In contrast to Viterbi, the PMAP inference 
(in the absence of constraints)
is by definition {\it point-wise}, i.e. symbol-by-symbol, hence not sensitive to the non-uniqueness issue. 

All in all, the main idea of \citep{Brushe1998} is to consider ``hybrid'' decoders 
that use intermediate values of the interpolation parameter $\mu$.  That is, the hybrid decoder 
with parameter $\mu$ is defined as a decoder that at time $t$ returns 
\begin{equation}\label{eqn:hybrids}
v_t=\arg\max_{i\in S}\{\alpha_t(i;\mu)\beta_t(i;\mu)\}, 
\end{equation}
provided some tie-breaking rule. 

Note also that in their attempt to hybridize PMAP with Viterbi in this manner, \citep{Brushe1998} 
instead of \eqref{eq:abp} use different transformations that are based on the 
following $(0,\infty)\to \mathbb{R}$ composite mapping
\begin{equation}\label{eq:brushe}
  F(\mu, d_1(\mu), d_2(\mu), \ldots, d_N(\mu)):=\frac{1+\left(N-1\right)\exp(-\mu)}{\mu}\log\left(\frac{1}{N}\sum_{j=1}^N\exp\left(\mu d_j\left(\mu\right)\right)
  \right),
\end{equation}
where $N=K$ (in our notation) and functions $d_j\left(\mu\right)$ are  continuous on $[0,\infty)$ with
finite limits $d_j(\infty)$ as $\mu\to \infty$. It is then not hard to verify that as $\mu\to 0$, 
the function 
\eqref{eq:brushe}
converges to $\sum_{j=1}^N d_j(0)$ (based on \citep[Proposition 1a]{Brushe1998}). 
At the same time, as $\mu\to \infty$ the same function converges to 
$\max_{1\le j\le N}\{d_j(\infty)\}$ (based on \citep[Proposition 1b]{Brushe1998}).  To establish the latter 
convergence, \citep{Brushe1998} refer to the Varadhan-Laplace Lemma, although the result is immediately seen
with just basic calculus, e.g. by using continuity of the logarithmic function, 
taking logarithm inside the limit in Proposition \ref{eqn:thermodynlim}, 
and identifying $a_j(\mu)$ with $e^{d_j(\mu)}$. 

This mapping is then applied recursively to $\alpha_{t}(i; \mu)$ and $\beta_{t}(i; \mu)$, the 
analogs of the forward and backward variables ($\kappa_{t}^\mu(i)$ and $\tau_{t}^\mu(i)$, respectively, 
in the notation of \citep{Brushe1998}),
to produce the correct end points/limits, i.e. PMAP and Viterbi/MAP (when the latter is unique). Specifically,
the transformed forward and backward variables would be re-defined as follows:
\begin{eqnarray}\label{eqn:brushesalphas}
\alpha_1(i; \mu)& := &\alpha_1(i);   \\
\alpha_{t}(i; \mu)& :=& 
\frac{1+\left(N-1\right)e^{-\mu}}{\mu}\log\left(\frac{1}{N}
\sum_{j=1}^Ne^{\mu \alpha_{t-1}\left(j;\mu\right)p_{ji}}\right)f_i\left(x_{t}\right),
~t=2,3\ldots, T; \nonumber\\
\beta_T(i; \mu)&:=&\beta_T(i)=1;\nonumber\\
\beta_{t}(i; \mu)& :=& 
\frac{1+\left(N-1\right)e^{-\mu}}{\mu}\log\left(\frac{1}{N}
\sum_{j=1}^Ne^{\mu \beta_{t+1}\left(j;\mu\right)p_{ij}f_j\left(x_{t+1}\right)}\right),
~t=T-1,T-2\ldots, 1. \nonumber
\end{eqnarray}
Above, we took the liberty to correct $\kappa_{1}^\mu(i)=\pi(i)$  ($\alpha_{1}(i; \mu)=\pi_i$ in our notation), 
which appears in \citep{Brushe1998} as 
equation (22) and also in the proofs of parts (a) and (b) of their Lemma 1.  
Clearly, in order for $\kappa_{1}^\mu(i)$ ($\alpha_1(i; \mu)$ in our notation) to match 
$\alpha_1(i)=P(Y_1=i,X_1=x_1)$ (as claimed in their Lemma 1), $\kappa_{1}^\mu(i)$ has to equal 
$\pi(i)b_i(O_1)$ (which is $\pi_i f_i(x_1)$ in our notation).  Note that 
equation (15) in \citep{Brushe1998} leaves $\alpha_1(i)$ undefined, but instead introduces $\alpha_0(i)$,
which is defined to be $\pi(i)$.  If that was an implicit intention to introduce a 
``silent'' state at $t=0$, then their equation (22) and the relevant parts of the proof of 
Lemma 1 would also have to start with $t=0$ and not with $t=1$. If, on the other hand, 
$t=0$ in equation (15) was simply a typing error and the intention was to have $t=1$, then 
the would-be definition of $\alpha_1(i)=\pi(i)$ contradicts an earlier equation just below 
their equation (14), which gives $\alpha_1(i)=P(O_1, q_1=S_i)=\pi(i)b_1(O_1)$ (i.e.  $P(Y_1=i,X_1=x_1)
=\pi_1 f_1(x_1)$ in our notation).  

Returning to the essence of the approach, note that the only reason stated in
\citep{Brushe1998} for choosing \eqref{eqn:brushesalphas} as the family of interpolating
transformations is the attainment 
of the required limits (i.e. PMAP when $\mu\to 0$, and Viterbi when $\mu\to \infty$). 
It is therefore not clear if \citep{Brushe1998}
realized that besides \eqref{eqn:brushesalphas}, there are 
other (single parameter) families of transformations, such as \eqref{eq:abp}, 
with the same limiting behavior. Naturally, the resulting interpolation generally depends on the 
choice of the transformations used.  In the absence of any special reason to use \eqref{eqn:brushesalphas},
\eqref{eq:abp} has an appeal of being more commonly used and looking simpler, should one really  
wish to pursue the idea of algorithmic hybridization.  Moreover,  we explain next (subsection \ref{sec:Brusheproblem}) 
{\it why the hybrid decoder defined by \eqref{eqn:hybrids} and the transformations \eqref{eqn:brushesalphas} does
not work in practice except with trivial examples}, and we also show (subsection \ref{sec:brushenotscales}) 
{\it how this decoder can be modified to become operational}. 
In contrast to this, we will show (subsection \ref{sec:powerscales}) that {\it the hybrid decoder  based on 
the transformations \eqref{eq:abp} becomes operational by modifying just the algorithm 
used for its computation, and not the  decoder}. This makes the transformations \eqref{eq:abp}
even more attractive as an alternative to \eqref{eqn:brushesalphas}.

\subsection{The hybrid decoder based on the transformations \eqref{eqn:brushesalphas} does
not work in practice except with trivial examples}\label{sec:Brusheproblem}
The key point is that the transform-based algorithmic hybridization attempts to compute
quantities which, at least for $\mu\approx 0$, are the same order of magnitude as the forward 
and backward probabilities $\alpha_t(i)=P(X^t=x^t,Y_t=i)$ and $\beta_t(i)=P(x_{t+1}^T|Y_t=i)$. 
These are well-known to vanish exponentially 
fast with $T$, see, for example, \citep[13.2.4]{Bishop2006} who also note that ``[f]or moderate
lengths of chain (say 100 or so), the calculation of the [$\alpha_t(j)$] will soon exceed the dynamic
range of the computer, even if double precision floating point is used.'' The situation clearly gets
worse as $\mu$ increases.  Indeed, recall \eqref{eqn:alphamuinf}, and note that  
$\max_{ s^{t}: s_t=i}p(x^t, s^t)=\alpha_t(i; \infty)\le  \sum_{ s^{t}: s_t=i}p(x^t, s^t)=\alpha_t(i)$ 
(which is also $\alpha_t(j; 1)$ in \eqref{eq:abp} and $\alpha_t(j; 0)$ in \eqref{eqn:brushesalphas}). 
This easily leads to a collapse of computations already with chains 
as short as $T=10$ (which indeed happens using the data and model from our experiments of
 Section \ref{sec:example} below).

It appears that the authors of \citep{Brushe1998} do not fully understand  the nature
of the above numerical problems when they
divert the reader's attention to the computation of the \texttt{logsumexp} function 
used in their transforms \eqref{eq:brushe}, \eqref{eqn:brushesalphas}. This is misleading as 
the $\log(e^a+e^b)=\max\{a,b\}+\log\left(1+e^{-|a-b|}\right)$ trick (alluded to  
by \citep{Brushe1998} in their Remark below equation (25))
is relevant to the problem of underflow only of the {\it intermediate values} (i.e. $e^a+e^b$ when
$a$ or $b$ is negative of a large magnitude, such as the logarithm of a very small probability).  
In the case of the transform \eqref{eq:brushe}, however, computations of the transformed, say, forward
variable $\alpha_{t+1}(i;\mu)$ \eqref{eqn:brushesalphas}, do require 
$\mu d_j(\mu)=\mu \alpha_t(j;\mu)p_{ji}$ and not its logarithm.  Thus, eventually (i.e. for some $t$) underflow occurs for some 
$\mu \alpha_t(i;\mu)p_{ji}$, and
then (i.e. for some possibly larger $t$) for all  $\mu \alpha_t(i;\mu)p_{ji}$. In terms of 
the \texttt{logsumexp} function, this means that  both $e^a$ and $e^b$ become 1 (and not zero!)
but the logarithm of their average (the core of  the transform \eqref{eq:brushe}) becomes 0, transferring
the underflow to the next generation, i.e. $\alpha_{t+1}(i;\mu)$.  Thus,  storing 
$\alpha_t(i;\mu)$ in the log-domain is irrelevant here since the transforms \eqref{eq:brushe}, 
\eqref{eqn:brushesalphas} with or
without the \texttt{logsumexp} trick, do
require the actual value of $\alpha_t(i;\mu)$.  Of course, one can conceivably introduce the \texttt{loglogsumexpexp}
function to operate on $\log(\alpha_t(i;\mu))$ and resolve this problem in that way, 
but it is not clear if the goal is worth the effort. 

Furthermore, insisting on that ``[t]he computational complexity and numerical implementation issues 
associated with the hybrid algorithm can be overcome using the Jacobian logarithm'', \citep[p. 3133]{Brushe1998} 
repeatedly refer to another paper, which proposes to compute the \texttt{logsumexp} function
$\log(\sum_k \exp(a_k))$ via recursive application of $\log(e^a+e^b)=\max\{a,b\}+\log\left(1+e^{-|a-b|}\right)$.
Although this recursive implementation should indeed be 
generally more accurate (albeit also computationally more expensive) than the commonly used
single-shift implementation
$\log(\sum_k \exp(a_k))=M+\log\left(\exp(a_k-M)\right)$ ($M=\max_k\{a_k\}$), 
as we just explained above, it is irrelevant to the real problem of computing the transformed 
forward and backward variables $\alpha_t(i,\mu)$, $\beta_t(i,\mu)$ ($\kappa^\mu_t(i)$, 
$\tau^\mu_t(i)$, respectively, in \citep{Brushe1998}). Thus, despite their claims, 
the approach of \citep{Brushe1998} {\it does not immediately provide an operational
decoding algorithm} except for trivially short chains. For example, using the two-state
HMM from the Example \ref{example:Brushe_norm} and the 64-bit MATLAB \citep{matlab_R2011b} 
(but without The Symbolic Math Toolbox)
installation on a (64-bit) Linux machine, the hybrid decoder based on \eqref{eqn:brushesalphas}
with $\mu=1$ already fails for $T=40$ (with or without the \texttt{logsumexp} trick).  
For comparison,  the hybrid decoder based on the power transform \eqref{eq:abp} ($\mu=1$) 
survives an order of magnitude longer.

A natural question is then whether the transform-based algorithmic hybridization approach
(using \eqref{eqn:brushesalphas} or \eqref{eq:abp}, or the like) can at all work in practice.   
The fact that no  such example has been given by \citep{Brushe1998}, or anyone else
uptodate, casts some doubt.  Below we 
give reassuring answers, which have been verified to work on several realistic examples.  

Indeed, it is well-known that in practice, to decode the $t$-th symbol the PMAP
decoder  uses the posterior probabilities $p_t(i|x^T)$ and not
the vanishing joint probabilities $p_t(i|x^T)p(X^T=x^T)=P(x^T,Y_t=i)=\alpha_t(i)\beta_t(i)$.
The posterior probabilities $p_t(i|x^T)$ are computed as 
$\tilde \alpha_t(i)\tilde \beta_t(i)$,
where $\tilde \alpha_t(i)=P(Y_t=i|x^t)$ and $\tilde \beta_t(i)=P(x_{t+1}^T|Y_t=i)/p(x_{t+1}^T|x^t)$
are the scaled analogs of the forward and backward probabilities  $\alpha_t(i)$ and
$\beta_t(i)$ \citep[13.2.4]{Bishop2006}.  This allows PMAP to bypass the 
aforementioned problem of numerical underflow. 

\subsection{The hybrid decoder \eqref{eqn:hybrids} is invariant 
to rescaling of the power-transformed \eqref{eq:abp} forward and backward variables $\alpha(\cdot; \mu)$,  
$\beta(\cdot; \mu)$.} \label{sec:powerscales}
Let us apply the same normalization approach to the transformed forward and backward variables,
first, using the power transform \eqref{eq:abp} and then \eqref{eqn:brushesalphas}. 
First, recall (e.g. \citep[13.2.4]{Bishop2006}) that $\tilde \alpha_t(i)$
are obtained 
by replacing the recursive definition
$$\alpha_{t}(i)=f_i(x_{t}) \sum_{j=1}^K \alpha_{t-1}(j)p_{ji}, \quad i=1,2,\ldots, K,$$
by the two-step self-normalized  definition  
\begin{eqnarray*}
p(x_{t}|x^{t-1})\tilde \alpha_{t}(i)&=&f_i(x_{t}) \sum_{j=1}^K \tilde \alpha_{t-1}(j)p_{ji},\quad i=1,2,\ldots, K,\nonumber \\ \nonumber
\tilde \alpha_{t}(i) &=& \frac{p(x_{t}|x^{t-1})\tilde \alpha_{t}(i)}{\sum_{s=1}^K p(x_{t}|x^{t-1})\tilde \alpha_{t}(s)}, \quad \text{for}
~t=2,\ldots,T,\\  \text{where} \quad \tilde\alpha_1(i)&=&\alpha_1(i)/c_1,~\text{and}~c_1:=p(x_1)=\sum_{s=1}^K\alpha_1(s).
\end{eqnarray*}
Thus, for all $t=2,3 \ldots T$, and for all $i=1,2,\ldots, K$, 
\begin{eqnarray*}\label{eqn:scaledalphas}
\tilde \alpha_{t}(i)&=&
 \frac{f_i(x_{t}) \sum_{j=1}^K \tilde \alpha_{t-1}(j)p_{ji}}{c_{t}},\quad\text{where, also according to \cite[(13.56)]{Bishop2006},}\\ \nonumber 
c_{t}&:=&p(x_{t}|x^{t-1})=\sum_{s=1}^K f_s(x_{t}) \sum_{j=1}^K \tilde \alpha_{t-1}(j)p_{js}. 
\end{eqnarray*}
Similarly, the rescaled backward variables are given by 
\begin{eqnarray*}
\tilde \beta_T(i)&:=&1;\nonumber\\
\tilde \beta_{t}(i)& :=&\frac{\sum_{j=1}^K
 p_{ij}f_j(x_{t+1})\tilde \beta_{t+1}(j)}
{c_{t+1}},\quad t=T-1,T-2,\ldots,1. \nonumber
\end{eqnarray*}

In the same manner, we  normalize the $\alpha_t(i; \mu)$ and $\beta_t(i;\mu)$ 
(defined by equations \eqref{eq:abp}) for any $\mu>0$
as follows:
\begin{eqnarray}\label{eq:abpp}
\tilde \alpha_1(i; \mu)& := &\alpha_1(i)/c_1(\mu)=\tilde \alpha_1(i),\quad\text{where}~c_1(\mu):=c_1~\text{for all}~\mu;  \\
\tilde \alpha_{t}(i; \mu)& :=& \frac{\left[\sum_{j=1}^K\left(\tilde \alpha_{t-1}(j; \mu)p_{ji}\right)^\mu\right]^\frac{1}{\mu }f_i(x_{t})}
{c_{t}(\mu)},\quad t=2,3,\ldots T; \nonumber\\
\tilde \beta_T(i; \mu )&:=&\beta_T(i)=1;\nonumber\\
\tilde \beta_{t}(i; \mu)& :=&\frac{\left[\sum_{j=1}^K
 \left(p_{ij}f_j(x_{t+1})\tilde \beta_{t+1}(j;
  \mu)\right)^\mu\right]^\frac{1}{\mu}}
{c_{t+1}(\mu)},\quad t=T-1,T-2,\ldots 1, \nonumber
\end{eqnarray}
where 
\begin{eqnarray*}
c_{t}(\mu)&:=&{\sum_{s=1}^K\left[\sum_{j=1}^K\left(\tilde \alpha_{t-1}(j; \mu)p_{js}\right)^\mu \right]^\frac{1}{\mu}f_s(x_{t})},
\quad t=2,3\ldots T. \nonumber\\
\end{eqnarray*}
Thus, $c_t(1)=c_t$ for all $t=1,2,\ldots,T$. Also note that, using induction on $t$ and \eqref{eqn:alphamuinf},
$\lim_{\mu\to 1} c_t(\mu)=c_t(1)$, and the limits $c_t(\infty):=\lim_{\mu\to \infty } c_t(\mu)$ exist and are finite for all $t=1,2,\ldots, T$.  
\begin{proposition}\label{prop:powertransformworks}
For  any $i\in S$,  we have 
\begin{enumerate}[label=\emph{\arabic*)}, leftmargin=0cm,itemindent=.5cm,labelwidth=\itemindent,labelsep=0.1cm,align=left]
 \item $\tilde \alpha_t(i; \mu)=\frac{\alpha_t(i; \mu)}{\sum_{s=1}^K \alpha_t(s; \mu)}=\frac{\alpha_t(i; \mu)}{\prod_{m=1}^t c_m(\mu)}$ for all $t=1,2,\ldots T$, 
and $\tilde \beta_t(i; \mu)=\frac{\beta_t(i; \mu)}{\prod_{m=t+1}^T c_m(\mu)}$ for all $t=1,2,\ldots T-1$ and 
for all $\mu>0$;
\item $\lim_{\mu\to 1}\tilde \alpha_t(i; \mu)=\tilde \alpha_t(i)$,
$\lim_{\mu\to 1}\tilde \beta_t(i; \mu)=\tilde \beta_t(i)$ for all $t=1,2,\ldots T$;
\item $\lim_{\mu\to \infty }\tilde \alpha_t(i; \mu)=\tilde \alpha_t(i;\infty):=\frac{\alpha_t(i;\infty)}{\sum_{s=1}^K\alpha_t(s,\infty)}$,
for all $t=1,2,\ldots T$, and $\lim_{\mu\to \infty }\tilde \beta_t(i; \mu)=:\tilde \beta_t(i;\infty)=\frac{\beta_t(i;\infty)}{\prod_{m=t+1}^T c_m(\infty)}$,
for all $t=1,2,\ldots T-1$, and, finally, $\lim_{\mu\to \infty }\tilde \beta_T(i; \mu)=:
\tilde \beta_T(i;\infty)=1$ trivially;
\item The hybrid decoder \eqref{eqn:hybrids} based on the transformations \eqref{eq:abp} and 
the hybrid decoder \eqref{eqn:hybrids} based on the transformations \eqref{eq:abpp} are 
one and the same decoder, provided that both use the same tie-breaking rule.
\end{enumerate}
\end{proposition}
\begin{proof}
The first claim concerning the $\tilde \alpha_t$ is trivially true for $t=1$ by definition of $\alpha_1(i; \mu)$, i.e. \eqref{eq:abp}.
Now, using induction on $t$, assume that the claim is true for $t-1$.
Write $a_{t-1}(\mu)$ for $(\sum_{s=1}^K \alpha_{t-1}(s; \mu))^{-1}$ so that $a_{t-1}(\mu)\alpha_{t-1}(j; \mu)=\tilde\alpha_{t-1}(j; \mu)$
and $a_{t-1}(\mu)=(\prod_{m=1}^{t-1} c_m(\mu))^{-1}$.   
Then, using \eqref{eq:abpp}, we get
 $$\tilde\alpha_{t}(i; \mu)=\frac{\left( \sum_{j=1}^K \left( a_{t-1}(\mu) \alpha_{t-1}(j; \mu) p_{ji}\right)^\mu\right)^{\frac{1}{\mu}}f_i(x_{t})}
{\sum_{s=1}^K\left( \sum_{j=1}^K \left( a_{t-1}(\mu) \alpha_{t-1}(j; \mu) p_{js}\right)^\mu\right)^{\frac{1}{\mu}}f_s(x_{t})},$$
which, upon cancellation of the $a_{t-1}(\mu)$, yields the required result 
$$
\frac{\left( \sum_{j=1}^K \left(  \alpha_{t-1}(j; \mu) p_{ji}\right)^\mu\right)^{\frac{1}{\mu}}f_i(x_{t})}
{\sum_{s=1}^K\left( \sum_{j=1}^K \left(  \alpha_{t-1}(j; \mu) p_{js}\right)^\mu\right)^{\frac{1}{\mu}}f_s(x_{t})}=\frac{\alpha_{t}(i; \mu)}{\sum_{s=1}^K 
\alpha_{t}(s; \mu)}.$$
To see that $\tilde \alpha_t(i; \mu)$ also equals $\frac{\alpha_t(i; \mu)}{\prod_{m=1}^t c_m(\mu)}$, write  
$$\tilde\alpha_{t}(i; \mu)=\frac{\left( \sum_{j=1}^K \left( a_{t-1}(\mu) \alpha_{t-1}(j; \mu) p_{ji}\right)^\mu\right)^{\frac{1}{\mu}}f_i(x_{t})}
{c_t(\mu)}=\frac{\left( \sum_{j=1}^K  \left(\alpha_{t-1}(j; \mu) p_{ji}\right)^\mu\right)^{\frac{1}{\mu}}f_i(x_{t})}
{(\prod_{m=1}^{t-1} c_m(\mu))c_t(\mu)},$$ which, recalling the original (unscaled) $\alpha_t(i;\mu)$  recursion, yields the result.

The $\beta$ variables are handled analogously.

The second claim is then a straightforward consequence of the first claim and the continuity 
(with respect to $\mu$, and in particular at $\mu=1$) of the power transform; for example, to establish
the result for the $\tilde \beta_t(i;\mu)$, observe that
$\prod_{m=t+1}^T c_m(\mu)\to \prod_{m=t+1}^T c_m(1)$ 
when $\mu\to 1$. 
The third claim also immediately follows from the first one and Proposition \ref{eqn:thermodynlim}, also noticing that
$\prod_{m=t+1}^T c_m(\mu)\to \prod_{m=t+1}^T c_m(\infty)$ as $\mu\to \infty$.
The fourth claim also immediately follows from the first claim  as  $v_t$ maximizes 
$\alpha_{t}(i; \mu)\beta_t(i; \mu)$ if and only if it maximizes $\tilde \alpha_{t}(i; \mu)\tilde \beta_t(i; \mu)$.
\end{proof}
In particular, we arrive at the following  characterization of the Viterbi
paths $v^T$, which is now possible to compute in practice for a wide range of models 
and parameters in contrast to the condition \eqref{eqn:viterbi-pointwise}:
\begin{corollary}\label{corollaryBrusheViterbi}
For any $t=1,2,\ldots, T$, 
$v_t=\arg\max_{i\in S}\{\tilde \alpha_t(i;\infty)\tilde \beta_t(i;\infty)\}$. 
\end{corollary}
Recall \eqref{eqn:pmaphybrids}, and thus note that the PMAP decoder also maximizes 
$\tilde \alpha_t(i;1)\tilde \beta_t(i;1)$. As a side note, consider also the following
decoder $v^T(x^T)$ that extrapolates the normalized power-transformed decoder to $\mu\to 0$, i.e.``beyond'' 
the PMAP decoding. Namely, for any $t=1,2,\ldots, T$, let 
$v_t=\arg\max_{i\in S}\{\tilde \alpha_t(i;0)\tilde \beta_t(i;0)\}$, where
for any $i\in S$, 
\begin{eqnarray}\label{eq:abpp0}
\tilde \alpha_1(i; 0)& := &\alpha_1(i)/c_1=\tilde \alpha_1(i);  \\
\tilde \alpha_{t}(i; 0)& :=& \frac{\left[\prod\limits_{j\in S_t(i)}\tilde \alpha_{t-1}(j; 0)p_{ji}\right]^{\frac{1}{K_t(i)}}f_i(x_{t})}
{\sum_{s=1}^K\left[\prod\limits_{j\in S_t(s)}\tilde \alpha_{t-1}(j; 0)p_{js}\right]^\frac{1}{K_t(s)}f_s(x_{t})},\quad
 t=2,3,\ldots T, \nonumber\\
\text{where}~S_t(i)&:=&\{j\in S:~ \tilde \alpha_{t-1}(j; 0)p_{ji}>0\}~\text{and}~K_t(i):=|S_t(i)|\nonumber \\ 
\tilde \beta_T(i; 0 )&:=&\beta_T(i)=1;\nonumber\\
\tilde \beta_{t}(i; 0)& :=&\frac{\left[\prod\limits_{j\in S^*_t(i)}
p_{ij}f_j(x_{t+1})\tilde \beta_{t+1}(j; 0)\right]^\frac{1}{K^*_t(i)}}
{\sum_{s=1}^K\left[\prod_{j\in S_{t+1}(s)}\tilde \alpha_{t}(j; 0)p_{js}\right]^\frac{1}{K_{t+1}(s)}f_s(x_{t+1})},\quad t=T-1,T-2,\ldots,1, \nonumber\\
\text{where}~S^*_t(i)&:=&\{j\in S:~ p_{ij}f_j(x_{t+1})\tilde \beta_{t+1}(j; 0)>0\}~\text{and}~K^*_t(i):=|S^*_t(i)|\nonumber.
\end{eqnarray}

\begin{corollary}\label{corollary:mu0}
Assume that $\lim_{\mu\to 0}\tilde \alpha_t(i;\mu)>0$ and
$\lim_{\mu\to 0}\tilde \beta_t(i;\mu)>0$ for all $i\in S$ and all $t=1,2,\ldots,T$.
Then $\tilde \alpha_{t}(i; 0)=\lim_{\mu\to 0}\tilde \alpha_t(i;\mu)$ and $\lim_{\mu\to 0}\tilde \beta_t(i;\mu)=\tilde \beta_{t}(i; 0)$ 
for all $i\in S$ and all $t=1,2,\ldots,T$, i.e. the decoder \eqref{eqn:hybrids} based on the transformations \eqref{eq:abpp} 
converges (upto the tie-breaking rule) to the decoder defined by \eqref{eq:abpp0} above.
\end{corollary}
\begin{proof}
This is a straightforward exercise in calculus, i.e. using continuity of the exponential function 
and invoking \citep[Proposition 1a]{Brushe1998}, with the positivity assumption making all $K_t(i)$ and $K^*_t(i)$ equal to $K$.
\end{proof}
Note also that the hybrid 
decoder \eqref{eqn:hybrids} based on the original, i.e. unnormalized variables \eqref{eq:abp},
generally does not have a limit as $\mu\to 0$. 

\subsection{Rescaling of the forward and backward variables $\alpha(\cdot; \mu)$ and $\beta(\cdot; \mu)$
defined by \eqref{eqn:brushesalphas} alters the hybrid decoder \eqref{eqn:hybrids}.} 
\label{sec:brushenotscales}
In the same manner as in \eqref{eq:abpp} above, we now normalize the 
$\alpha(\cdot; \mu)$ and $\beta(\cdot; \mu)$ variables transformed according to \eqref{eqn:brushesalphas}.
Thus, for any $\mu>0$ and for any $i\in S$, let 
\begin{eqnarray}\label{eq:brusherenormed}
\check \alpha_1(i; \mu)& := &\alpha_1(i)/\sum_{s=1}^K\alpha_1(s)=\tilde \alpha_1(i);  \\
\check \alpha_{t}(i; \mu)& :=& \frac{\log\left[\frac{1}{K}\sum_{j=1}^Ke^{\mu \check \alpha_{t-1}(j; \mu)p_{ji}}\right]f_i(x_{t})}
{\sum_{s=1}^K\log\left[\frac{1}{K}\sum_{j=1}^Ke^{\mu \check \alpha_{t-1}(j; \mu)p_{js}} \right]f_s(x_{t})},\quad t=2,3,\ldots T; \nonumber\\
\check \beta_T(i; \mu )&:=&\beta_T(i)=1,\quad  t=T-1,T-2,\ldots,1;\nonumber\\
\check \beta_{t}(i; \mu)& :=&\frac{\log\left[\frac{1}{K}\sum_{j=1}^K
 e^{\mu p_{ij}f_j(x_{t+1})\check \beta_{t+1}(j; \mu)}\right]}
{\sum_{s=1}^K\log\left[\frac{1}{K}\sum_{j=1}^Ke^{\mu \check \alpha_{t}(j; \mu)p_{js}} \right]f_s(x_{t+1})},\quad t=T-1,T-2,\ldots 1. \nonumber
\end{eqnarray}

\begin{proposition}\label{prop:brushetransformworks}
For  any $i\in S$,  we have 
\begin{enumerate}[label=\emph{\arabic*)},leftmargin=0cm,itemindent=.5cm,labelwidth=\itemindent,labelsep=0.1cm,align=left]
\item $\lim_{\mu\to 0}\check \alpha_t(i; \mu)=\tilde \alpha_t(i)$,
$\lim_{\mu\to 0}\check \beta_t(i; \mu)=\tilde \beta_t(i)$ for all $t=1,2,\ldots T$;
\item $\lim_{\mu\to \infty }\check \alpha_t(i; \mu)=\tilde \alpha_t(i;\infty)$
and $\lim_{\mu\to \infty }\check \beta_t(i; \mu)=\tilde \beta_t(i;\infty)$,
for all $t=1,2,\ldots T$.
\item The hybrid decoder \eqref{eqn:hybrids} based on the transformations \eqref{eqn:brushesalphas} and 
the hybrid decoder \eqref{eqn:hybrids} based on the transformations \eqref{eq:brusherenormed} are 
generally different, even if both use the same tie-breaking rule.
\end{enumerate}
\end{proposition}
\begin{proof} The first two claims are straightforward extensions of Lemmas 1 and 2 of 
\citep{Brushe1998}.
To see this, first  restore the previously reduced factor $\frac{1+(K-1)e^{-\mu}}{\mu}$
in both the numerator and denominator of the expressions for $\check \alpha_{t}(i; \mu)$ and
$\check \beta_{t}(i; \mu)$. Then apply induction on $t$  
(first in the forward manner for the $\alpha$ variables
and then backward for the $\beta$ variables).  For example, assume  that
$\lim_{\mu\to \infty }\check \beta_{t+1}(i; \mu)=\tilde \beta_{t+1}(i;\infty)$.
Then, as $\mu\to \infty$, 
\begin{eqnarray*}
 \frac{1+(K-1)e^{-\mu}}{\mu}\log\left[\frac{1}{K}\sum_{j=1}^K e^{\mu p_{ij}f_j(x_{t+1})\check \beta_{t+1}(j; \mu)}\right]
&\to &\max_{j\in S}\left(p_{ij}f_j(x_{t+1})\tilde \beta_{t+1}(j; \infty)\right),
\end{eqnarray*}
which is, according to claim 3 of Proposition \ref{prop:powertransformworks},
\begin{eqnarray*}
\max_{j\in S}\left(p_{ij}f_j(x_{t+1})\beta_{t+1}(j; \infty)/\prod\limits_{m=t+2}^T c_m(\infty)\right)
&=&\max_{j\in S}\left(p_{ij}f_j(x_{t+1})\beta_{t+1}(j; \infty)\right)/\prod\limits_{m=t+2}^T c_m(\infty)
\end{eqnarray*}
Next, recalling \eqref{eqn:alphamuinf}, we get that the numerator in the expression for  $\lim_{\mu\to \infty }\check \beta_t(i; \mu)$
is given by 
$\beta_{t}(i; \infty)/\prod_{m=t+2}^T c_m(\infty)$.  Observing that the denominator is given by
\begin{eqnarray*}
\lim\limits_{\mu\to\infty}
\frac{1+(K-1)e^{-\mu}}{\mu}\sum_{s=1}^K\log\left[\frac{1}{K}\sum_{j=1}^Ke^{\mu \check \alpha_{t}(j; \mu)p_{js}} \right]f_s(x_{t+1})
&=&\sum_{s=1}^K \max_{j\in S} \left(\tilde\alpha_t(j;\infty)p_{js}\right) f_s(x_{t+1}),
\end{eqnarray*}
which is just $c_{t+1}(\infty)$, finally
gives $\lim_{\mu\to \infty }\check \beta_t(i; \mu)=\beta_{t}(i; \infty)/\prod_{m=t+1}^T c_m(\infty)=\tilde\beta_t(i; \infty)$,
as required.

As a counter-example proving the last claim, consider the simple HMM
from \citep[p.1840]{matlab_stat}. 
\begin{example}\label{example:Brushe_norm}
Let $S=\{1,2\}$ and let $\{1,2,\ldots,6\}$ be the emission alphabet. Let the initial
distribution $\pi$, transition probability matrix $\mathbb{P}$, and the emission 
distributions $f_s$, $s\in S$, be defined as follows: 
\begin{eqnarray*}
\pi=\begin{pmatrix}
     2/3 \\1/3
    \end{pmatrix},
&  \mathbb{P}=\begin{pmatrix}
              0.95  &  0.05 \\
    0.1   & 0.9\end{pmatrix},~\pi^t \mathbb{P}=\pi^t,~
& 
\begin{array}{lllllll}
&1   & 2     & 3     & 4 & 5 & 6 \\
f_1(\cdot) &1/6 & 1/6  & 1/6 & 1/6 & 1/6 & 1/6 \\
f_2(\cdot)&0.1 &0.1  &0.1  &0.1 & 0.1 &0.2 
\end{array}.
\end{eqnarray*}
Suppose $x^{5}=(2,6,6,4,1)$ has been observed. Take $\mu=7$. Tables \ref{tab:Brushe} show
outputs of the original (top) and  normalized (bottom)
transformed decoders, respectively.  Clearly, the decoders return different paths.

\begin{table}[htb]
 \centering
 \begin{center}
\begin{tabular}{|l|llllrr|} \hline
t& $\alpha_t(1;\mu)$&$\beta_t(1;\mu)$ & $\alpha_t(2;\mu)$ & $\beta_t(2;\mu)$ & $\alpha_t(1;\mu)\beta_t(1;\mu)$ & 
$\alpha_t(2;\mu)\beta_t(2;\mu)$ \\        \hline \hline
 &             &             &              &            &   $10^{-6}$       &   $10^{-6}$  \\
1&      0.11111&   6.6968e-05&     0.033333 &  0.00019826&   \bf{7.4409} &  6.6088     \\
2&     0.010576&   0.00071029&    0.0091583 &  0.00085352&   7.5121      &\bf{7.8168}     \\
3&    0.0009266&    0.0083987&    0.0022209 &    0.003471&   \bf{7.7823} &  7.7088           \\
4&    9.201e-05&      0.10141&   0.00010268 &    0.058041&   \bf{9.3311} &  5.9598             \\
5&   8.1481e-06&            1&   4.8559e-06 &           1&   \bf{8.1481} &  4.8559     \\ \hline
\end{tabular} \\                                                              
 \begin{tabular}{|l|llllrr|} \hline
t& $\check\alpha_t(1;\mu)$&$\check\beta_t(1;\mu)$ & $\check\alpha_t(2;\mu)$ & $\check\beta_t(2;\mu)$ & $\check\alpha_t(1;\mu)\check\beta_t(1;\mu)$ & 
$\check\alpha_t(2;\mu)\check\beta_t(2;\mu)$ \\        \hline \hline
1&       0.76923 &     0.30879&      0.23077 &     0.97296 &     \bf{0.23753} &     0.22453          \\
2&       0.58963 &     0.55137&      0.41037 &     0.55227 &     \bf{0.32510} &     0.22664           \\
3&       0.35383 &     1.15172&      0.64617 &     0.39942 &     \bf{0.40751} &     0.25809          \\
4&       0.46886 &     1.03712&      0.53114 &     0.59356 &      \bf{0.48626} &    0.31526           \\
5&       0.60611 &           1&      0.39389 &           1 &     \bf{0.60611} &     0.39389          \\   \hline
\end{tabular}                                                               
 \end{center}
\caption{$\mu=7$. Top: Output from the original (unnormalized) transformed decoder based on the transformations \eqref{eqn:brushesalphas};
the optimal path is $(1,2,1,1,1)$.  Bottom: Output from the normalized transformed decoder based on the transformations \eqref{eq:brusherenormed};
the optimal path is $(1,1,1,1,1)$.} \label{tab:Brushe}
\end{table}
\end{example}
\end{proof}
Note that unlike the normalized hybrid decoder based on the power-transform, this normalized hybrid
decoder generally does not satisfy the first claim of Proposition \ref{prop:powertransformworks}.
(Indeed, satisfying these conditions would contradict the third claim of the latter
Proposition \ref{prop:brushetransformworks}.)


We have also experimented with these normalized hybrid decoders
using a subset of real data (and a realistic HMM with $K=6$ states) from our experimental 
Section \ref{sec:example} and can indeed
confirm convergence of the hybrid decoder based \eqref{eq:brusherenormed} to the PMAP decoder 
with $\mu=0.001$ and to the Viterbi decoder with $\mu=10000$ 
for sequences of length $T=100$. Naturally, the above range of $\mu$ values would generally need to increase 
significantly with $T$.

Below, we summarize our views on the idea of purely algorithmic hybridization of MAP and PMAP.
\begin{enumerate}[leftmargin=0cm,itemindent=.5cm,labelwidth=\itemindent,labelsep=0.1cm,align=left]
\item The method presented in \citep{Brushe1998} need not work, i.e. 
can fail to converge to the Viterbi path, when 
the Viterbi path is not unique, cf. Example \ref{example:BrusheMAP} above.
\item Since the method depends on the transformation used,
more work may be needed to understand which (if any) particular 
transformation/interpolation could be suitable for a specific application;
the choice of \eqref{eqn:brushesalphas} made in \citep{Brushe1998} seems to be rather arbitrary.
\item Also, the choice of \eqref{eqn:brushesalphas} does not work in practice 
except with trivially short sequences; the underlying transformations
can be normalized but this alters the decoder (Proposition \ref{prop:brushetransformworks}).
The choice of \eqref{eq:abp} is better in several aspects, mainly for its
rescaling property (subsection \ref{sec:powerscales}), i.e. the decoder is  indeed
ready to work in practice. 
\item 
Algorithmically defined estimators are notoriously hard to analyze
analytically \citep[pp. 25, 129-131]{Winkler2003}. Indeed, it is not 
clear if the general members of the above
interpolating families (regardless of the transformation used) satisfy any 
explicit optimality criteria; 
this makes it difficult to interpret such decoders.
This may also discourage the use of such decoders in
more complex inference cycles (i.e. when any genuine model parameters are
to be estimated as well, e.g. Viterbi Training \citep{koski, AVT4, AVTproof}).
\item The point-wise hybridization scheme \eqref{eqn:hybrids}
can itself be altered. For example, other recursion schemes (for
  example, cf. \citep[pp. 272-273]{koski} for Derin's formula) can also be
  applied for this purpose.  However, now more than a decade after the appearance of \citep{Brushe1998}, we are not aware
of any practical application of the idea of algorithmic hybridization of the MAP-PMAP inferences.
Besides the plausible reasons already discussed in Subsection \ref{sec:pathdiffers}
(that actually extend to any type of MAP-PMAP hybridization),
we suspect that this particular type of hybridization has not seen application
mostly because of the lack of interpretation of its solutions and the aforementioned
flaws in the original work \citep{Brushe1998} introducing the idea.
\end{enumerate}

\section{Experiments}\label{sec:example}
We illustrate the performance of the Viterbi, PMAP, and some of the other known and new decoders 
on the task
of predicting protein secondary structure in single amino-acid sequences. 
For this illustration purpose, our decoders are based entirely on the ordinary first order HMM. In particular,
when decoding an amino-acid sequence, they do not use cues from decoded homologous sequences
(other than by allowing homologous sequences to be part of the training set for estimation 
of the model parameters).
Certainly, successful predictors in practice are significantly more elaborate, in particular, 
do exploit intensively information from decoded homologs, and also
include interactions at ranges considerably longer than that of the first order HMM 
\citep{protein_secondary_struct_predict_2006_borodovsky}.  However, 
our current goal is not to attain the absolute record on the task (which, not so long ago, was 
reported to be about 70\% \citep{protein_secondary_struct_predict_2006_borodovsky}), but to merely emphasize
the following two points.  First, the difference in performance between the Viterbi and PMAP decoders can be 
appreciable in practice already with the ordinary first order HMMs having as few as six hidden states. Secondly,
using the new family of decoders (i.e. solutions to the generalized risk minimization 
\eqref{gen-problem} and  \eqref{gen-L1-pen})
gives a potentially useful additional flexibility by exercising
trade-offs between  principled performance measures (subsection \ref{sec:moremotivation}). 

Our data are a non-redundant subset 
of the {\it Protein Data Bank} \citep{Berman01012000}.  Specifically, the secondary
structural elements have been found from their atomic coordinates using  \texttt{SSENVID} \citep{SSENVID}
and the resulting data can be freely downloaded from 
\href{http://personal.rhul.ac.uk/utah/113/VA/env_seqssnr.txt}
{http://personal.rhul.ac.uk/utah/113/VA/env$\_$seqssnr.txt}.
The data contain $N=25713$ realizations $(x^{T_n}(n),y^{T_n}(n))$, $n=1,2,\ldots,N$,  
with three original hidden states $\{a,b,c\}$, representing
$\alpha-$helix, $\beta-$strand, and coil, respectively. The average length
$\bar T$ of a realization is 167 positions. The observations $x^{T_n}(n)$
come from a 20 symbol
emission alphabet of amino-acids $$\mathcal{X}=\{A,C,D,E,F,G,H,I,K,L,M,N,P,Q,R,S,T,V,W,Y\}.$$ 
We further distinguish
four subclasses of the $\alpha$-helix class $a$. The definition and enumeration of the
final six classes are as follows: Class one consists of the short, up to seven $a$ long, $\alpha$-helices.
Classes two and three consist of the $\beta$-strands (any number of $b$'s) and coil sequences (any number
of $c$'s), respectively. Classes four, five, and six derive from the $a$'s that 
comprise an $\alpha$-helix of length at least eight, thereafter referred to as long. Specifically, class four 
is the so-called $N$-end, which is the first four $a$'s of a long $\alpha$-helix.
Similarly, class six is the so called $C$-end, which is the last four $a$'s of a long $\alpha$-helix.
Any $a$'s in the middle of a long $\alpha$-helix are class five. Refining the original classification 
has been known to improve prediction of protein secondary structure \citep{Salamov199511}. For simplicity,
here we only sub-divide the $\alpha$-helix class (whereas \citep{Salamov199511} go further) 
given the limited goals of these experiments.    

The (maximum likelihood estimates of the) transition and emission distribution matrices as well 
as the vector of the initial probabilities computed from 
{\it all of the realizations} are given in Appendix \ref{sec:appendix1}.  

The following experiments emulate a typical practical situation
by re-estimating these parameters from $N-1$ sequences and
using the re-estimated values to decode a remaining sequence. 
We repeat the process $N$ times in the leave-one(sequence)-out fashion. 
We do not  impose stationarity in these
experiments as we did not have any prior evidence of stationarity. 
Indeed,  the (estimated) initial distribution 
$\hat \pi$ appears to be very different from the stationary one ($\hat \pi_{inv}$, see Appendix \ref{sec:appendix1}) 
and many sequences in the dataset are quite short. 

Figure \ref{fig:Case877} displays case 877, which is 149 positions long and is split into two pieces 
at position $t=75$ (shown in both images).  The ground truth is shown by the top (0) row.
This case is typical in several senses.  First, in this case the PMAP decoder (row 2) shows the median gain in accuracy (of
about 11\%) over the Viterbi decoder (row 1); see subsequent subsections for a discussion of
performance measures.  Secondly, the PMAP, i.e. optimal accuracy output, is inadmissible in this case, which 
is evident from, e.g. the isolated state five (yellow) island (transitions between states three and
five are forbidden).  Rows 3 through 5 are outputs from 
the PVD, Constrained PMAP, and Rabiner $k=2$ decoders, respectively. It is typical of the PVD and Constrained PMAP
decoders to tie. Outputs from other members of the generalized posterior Viterbi 
\eqref{gen-problem} and PMAP \eqref{L1-gen-problem} hybrid decoders are given in rows  6-18, and 19-31, respectively.
Table \ref{tab:case877} gives a detailed legend for interpreting the outputs.   The monotonicity 
of the generalized PVD hybrid inference (Corollary \ref{corollary:optimization}, part 1, and Corollary \ref{corollary}, inequalities
\eqref{eqn:increase} and \eqref{eqn:increase2}) is illustrated by following 
the posterior risk columns $\bar R_\infty$ and $\bar R_1$   across
rows 2 (PMAP), then 6 through 17, and finally 1 (Viterbi); 
PVD (row 3) is attained when $\alpha\approx 0$ (rows 6-9) 
and here is also indistinguishable from Constrained PMAP (row 4).  
The monotonicity 
of the generalized PMAP hybrid inference (Corollary \ref{corollary:optimization}, part 3) is illustrated by following 
the $\bar R_\infty$ and $R_1$ columns across
rows 2 (PMAP), then 19 through 30, and finally 1 (Viterbi); 
Constrained PMAP (row 4) is attained when $\alpha\approx 0$ (rows 19-20) 
and here is also indistinguishable from PVD (row 3).  

Note how the decoder in row 16 (Figure \ref{fig:Case877}) differs from its neighbors, 
specifically, how it completely misses the terminal activity, which is to a variable
extent captured by both  ``more accurate'' (row 15) and  ``more probable'' (row 17) neighbors of this decoder.
In practice, application specific performance measures would likely to be of more interest than 
the simple measures used here for illustration of the ideas (see also Section \ref{sec:discussion}).

Rows 18 and 31 are the ``data blind'' {\it maximum a priori}  and {\it pointwise maximum a priori} 
decodings, which are members of both the generalized hybrid families. 
These decodings tie not only in this  but in all the other cases as well;  see the structure of
the (overall) transition matrix $\mathbb{P}$ in Appendix \ref{sec:appendix1} to understand the overwhelming
dominance of class 3 (``coil'') in the absence of the amino-acid information.

\begin{figure}[ht!]
 \centering
  \includegraphics[trim = 0 -19mm 0 0, clip, height=0.351\textheight]{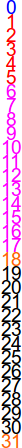}\includegraphics[width=\textwidth]{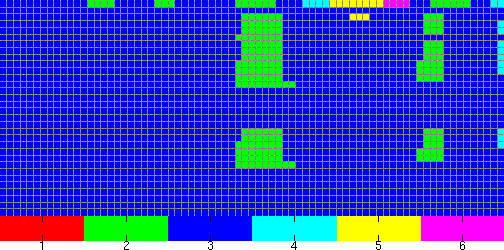}\includegraphics[trim = 0 -19mm 0 0, clip, height=0.351\textheight]{enumeration.png}\\
\includegraphics[height=0.3\textheight]{enumeration.png}\includegraphics[trim = 0 0 0 0.047\textheight, clip, width=\textwidth]{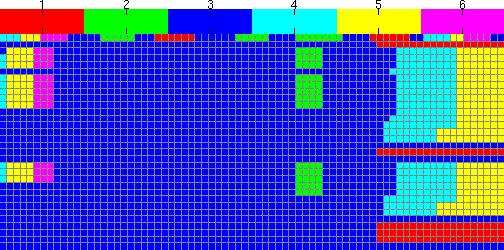}\includegraphics[height=0.3\textheight]{enumeration.png}
 \caption{Performance of some of the well-known and new decoders on Case 877. For legend, see Table \ref{tab:case877}.}
 \label{fig:Case877}
\end{figure}

\begin{table}
\hspace*{-2cm}\begin{tabular}{|l||l|l|l|l|l|l|l|l|l|l|l|}\hline 
R & \multicolumn{7}{|c|}{Output $\hat y^{149}(877)$}                    & Empir.  &  \multicolumn{3}{|c|}{posterior}     \\
 \cline{2-8}
o    & \multicolumn{2}{|c|}{Generalized}&Alias& $C_1$        & $C_2$& $C_3$& $C_4$&error & \multicolumn{3}{|c|}{ risks}    \\ \cline{2-3}\cline{10-12}
w    & PVD &  PMAP &  & & & & &rate(\%) &$\overline{R}_\infty$ &$\overline{R}_1$ & $R_1$(\%)\\     \hline\hline
0   & \multicolumn{7}{|c|}{Truth}                                                       &  0                           & 0.4907  &  1.1311 & 59.2173  \\\hline
1   &+&+&Viterbi& $1-\alpha=\frac{1}{k}=0$       & $\alpha=1$       &  0   &  0         & \textcolor{magenta}{56.3758} & \textcolor{blue}{0.1604} &  \textcolor{red}{0.8296} & \textcolor{red}{50.3368}          \\ \hline
2   &+&+&PMAP   & $1-\alpha=\frac{1}{k}=1$       & $\alpha=0$       &  0   &  0         & \textcolor{blue}{45.6376}    & \textcolor{red}{$\infty$}&  \textcolor{blue}{0.6905} &\textcolor{blue}{46.7752}           \\ \hline
3   &+& &PVD   & $1-\alpha=\frac{1}{k}\approx 1$& $\alpha\approx 0$&  0   &  0          & 46.9799                      &  \textcolor{magenta}{0.2486} & \textcolor{cyan}{0.6961}  &\textcolor{cyan}{46.9188}            \\
    &+& &      & $\approx 1$                    &0                 &  0   & $\approx 0$ &                              &          &         &            \\\hline
4   & & +&Constr.& $\approx 1$                    &$\approx 0$       &0     &  0        & 46.9799                      &  \textcolor{magenta}{0.2468} & \textcolor{cyan}{0.6961}  &\textcolor{cyan}{46.9188}             \\
    & &  &  PMAP&                                &                  &      &            &                              &          &         &                  \\ \hline
5   &+&  &Rabiner   &  n/a               & n/a                &n/a    & n/a             & 53.0201                      &   0.1823& 0.7118  &47.4429              \\
    & &  & $k=2$    &      &                  &      &                                  &                              &          &         &           \\\hline
6   &+&  &   & $1-\alpha=\frac{1}{k}=0.999$ & $\alpha=0.001 $&  0   &  0   &    46.9799                                &  \textcolor{magenta}{0.2486} & \textcolor{cyan}{0.6961}  &\textcolor{cyan}{46.9188}              \\    
7   &+&  &   & $1-\alpha=\frac{1}{k}=0.995$ & $\alpha=0.005 $&  0   &  0   &    46.9799                                &  \textcolor{magenta}{0.2486} & \textcolor{cyan}{0.6961}  &\textcolor{cyan}{46.9188}             \\    
8   &+&  &   & $1-\alpha=\frac{1}{k}=0.990$ & $\alpha=0.010 $&  0   &  0   &    46.9799                                &  \textcolor{magenta}{0.2486} & \textcolor{cyan}{0.6961}  &\textcolor{cyan}{46.9188}              \\    
9   &+&  &   & $1-\alpha=\frac{1}{k}=0.950$ & $\alpha=0.050 $&  0   &  0   &    46.9799                                &  0.2352 & 0.6964  &46.9322             \\    
10  &+&  &   & $1-\alpha=\frac{1}{k}=0.900$ & $\alpha=0.100 $&  0   &  0   &    46.9799                                &  0.2352 & 0.6964  &46.9322             \\    
11  &+&  &   & $1-\alpha=\frac{1}{k}=0.(6)$ & $\alpha=0.(3) $&  0   &  0   &    53.0201                                &  0.1897 & 0.7065  &47.2499             \\    
12  &+&  &   & $1-\alpha=\frac{1}{k}=0.500$ & $\alpha=0.500 $&  0   &  0   &    54.3624                                &  0.1791 & 0.7142  &47.5372             \\    
13  &+&  &   & $1-\alpha=\frac{1}{k}=0.(3)$ & $\alpha=0.(6) $&  0   &  0   &    \textcolor{magenta}{56.3758}           &  0.1700 & 0.7277  &48.0356            \\    
14  &+&  &   & $1-\alpha=\frac{1}{k}=0.250$ & $\alpha=0.750 $&  0   &  0   & \textcolor{red}{57.0470}                  &  0.1680 & 0.7331  &48.1738            \\    
15  &+&  &   & $1-\alpha=\frac{1}{k}=0.200$ & $\alpha=0.800 $&  0   &  0   & \textcolor{red}{57.0470}                  &  0.1680 & 0.7331  &48.1738            \\    
16  &+&  &   & $1-\alpha=\frac{1}{k}=0.100$ & $\alpha=0.900 $&  0   &  0   &\textcolor{red}{57.0470}                   &  \textcolor{cyan}{0.1645} & \textcolor{magenta}{0.7637}  &\textcolor{magenta}{48.9620}           \\    
17  &+&  &   & $1-\alpha=\frac{1}{k}=0.010$ & $\alpha=0.990 $&  0   &  0   &    \textcolor{magenta}{56.3758}           &  \textcolor{blue}{0.1604} & \textcolor{red}{0.8296}  &\textcolor{red}{50.3368}           \\ 
\hline   
18  &+&+& MA-   &      0                       &      0         &  0    & 1& \textcolor{red}{57.0470}                  &  \textcolor{cyan}{0.1645} & \textcolor{magenta}{0.7637}  &\textcolor{magenta}{48.9620}           \\
  &&& Prior   &                          &             &     &    &                                                    &          &         &             \\ \hline   
19  & &+&   & 0.999                        &      0.001     &  0   &  0   &    46.9799                                 & \textcolor{magenta}{0.2486}  & \textcolor{cyan}{0.6961}  &\textcolor{cyan}{46.9188}         \\            
20  & &+&   & 0.995                        &      0.005     &  0   &  0   &    46.9799                                 &  \textcolor{magenta}{0.2486} & \textcolor{cyan}{0.6961}  &\textcolor{cyan}{46.9188}            \\            
21  & &+&   & 0.990                        &      0.010     &  0   &  0   &    \textcolor{cyan}{46.3087}               &  0.2417 & 0.6962  &46.9245          \\            
22  & &+&   & 0.950                        &      0.050     &  0   &  0   &    50.3356                                 &  0.2009 & 0.7021  &47.0773              \\            
23  & &+&   & 0.900                        &      0.100     &  0   &  0   &    50.3356                                 &  0.2009 & 0.7021  &47.0773             \\            
24  & &+&   & 0.(6)                        &      0.(3)      &  0   &  0   &   54.3624                                 &  0.1776 & 0.7165  &47.6139             \\            
25  & &+&   & 0.500                        &      0.500     &  0   &  0   &    \textcolor{red}{57.0470}                &  0.1680 & 0.7331  &48.1738             \\            
26  & &+&   & 0.(3)                        &      0.(6)      &  0   &  0   &   \textcolor{red}{57.0470}                &  0.1680 & 0.7331  &48.1738            \\            
27  & &+&   & 0.250                        &      0.750     &  0   &  0   &    \textcolor{red}{57.0470}                &  \textcolor{cyan}{0.1645} & \textcolor{magenta}{0.7637}  &\textcolor{magenta}{48.9620}           \\            
28  & &+&   & 0.200                        &      0.800     &  0   &  0   &    \textcolor{magenta}{56.3758}            &  \textcolor{blue}{0.1604} & \textcolor{red}{0.8296}  &\textcolor{red}{50.3368}         \\            
29  & &+&   & 0.100                        &      0.900     &  0   &  0   &    \textcolor{magenta}{56.3758}            &  \textcolor{blue}{0.1604} & \textcolor{red}{0.8296}  &\textcolor{red}{50.3368}            \\            
30  & &+&   & 0.010                        &      0.990     &  0   &  0   &    \textcolor{magenta}{56.3758}            &  \textcolor{blue}{0.1604} & \textcolor{red}{0.8296}  &\textcolor{red}{50.3368}             \\ \hline           
31  &+&+&PMA-&   0                          &       0        &  1    & 0  &    \textcolor{red}{57.0470}                &  \textcolor{cyan}{0.1645} & \textcolor{magenta}{0.7637}  &\textcolor{magenta}{48.9620}            \\ 
  &&&   Prior&                             &             &     &  &                                                    &          &         &            \\   \hline         
\end{tabular}                                                                                                    
\caption{Case 877. Performance of the well-known and some of
the new decoders.  Worst, second worst, best and second best entries in each 
category are highlighted
in red, magenta, blue and cyan respectively. \label{tab:case877}}
\end{table}

Additionally, instead of using the
actual data, we simulate synthetic datasets each of which having the same number  $N=25713$
of sequences,  in the following way. 
Let $\{\hat \pi_{sn}\}_{s\in S}$, $\widehat{\mathbb{P}}_n$, $\{\widehat{P}_{sn}, s\in S\}$ be the estimates
of the  HMM parameters (initial, transition, and emission distributions, respectively) 
obtained from $(x^{T_n}(n),y^{T_n}(n))$, the $n$-th actual realization. Then the $n$-th 
simulated realization is a sample 
of length $T_n$ from the (first order homogeneous) HMM with these parameters (note
that the initial distributions $\{\hat \pi_{sn}\}_{s\in S}$ are necessarily degenerate).  

\subsection{Performance measures and their estimation}
Given a classification method $g$, our principal performance measures are the $R_1(g)$ 
risk $\mathbb{E}R_1(g(X^T)|X^T)$ 
(see \eqref{eqn:R1}) and the $\R_\infty$ risk
$\mathbb{E}\R_\infty(g(X^T)|X^T)$, \eqref{loglikerisk} (note that 
it is not practical to operate with $R_\infty$ \eqref{eqn:Rinf} since it is virtually 1
for reasonably long realizations).  For the $\R_\infty$ results, see Subsection 
\ref{sec:pmapvsViterbibyprob}.

The $R_1$ risk is simply the point-wise error rate $\frac{1}{T}\sum_{t=1}^TP(\hat Y_t\neq Y_t)$,
where $\hat Y$ stands for $g(X^T)$. This assumes $T$ to be non-random; more generally, 
$T$ is  random and the $R_1$ risk is then given by
$\mathbb{E}_T\left[\frac{1}{T}\sum_{t=1}^TP\left(\hat Y_t\neq Y_t|T\right)\right]$. We refer to $1-R_1$ 
as {\it accuracy} when comparing our decoders (e.g. Subsection 
\ref{sec:pmapvsViterbi} below).  Note that given a decoder $g$, $R_1(g)$, 
is simply a parameter of the underlying population of all $(T,x^T,y^T)$ that could
potentially be observed.
If the current hidden Markov model were not too crude for this population, we would compute such risks 
if not analytically, then at least by using Monte-Carlo simulations, for any $g$ of interest.
In reality, however, we need to estimate them. The situation is further complicated by the
fact that the classification method $g$ is specified only up to the model parameters, which are  unknown 
and are estimated from the data. 

All in all, we use the usual cross-validation (CV) estimation. Specifically, to decode
$x^{T_n}(n)$, $g$ uses the estimates of the parameters obtained from the remaining
$N-1$ sequences.  Thus, if $g$ outputs $\hat y^{T_n}$, then we take the empirical point-wise error rate
\begin{eqnarray}\label{eq:en}
 \hat e_n=\frac{1}{T_n} \sum_{t=1}^{T_n} \mathbb{I}_{\{\hat y_t\neq y_t(n)\}}
\end{eqnarray}
to be an estimate of $R_1(g)$. Clearly, if $g$ used the same fixed parameters as 
used in the definition of $R_1(g)$, then 
$\mathbb{E}[\hat e_n]=R_1(g)$, i.e. $\hat e_n$ would be unbiased for $R_1(g)$, and
so would be the average 
\begin{eqnarray}\label{eq:ecv}
\hat e_{CV}=\frac{1}{N} \sum\limits_{n=1}^N \hat e_n.
\end{eqnarray}
Obviously, in reality $\hat e_{CV}$ is likely to be biased. For this reason we also look
at the model-based CV estimate of $R_1$ given by
\begin{eqnarray}\label{eqn:cvestimate}
\hat R_1=\frac{1}{N} \sum\limits_{n=1}^N R_1(\hat y^{T_n}|x^{T_n}(n)).
\end{eqnarray}
Computation of $R_1(\cdot |x^T)$ indeed relies on the model being correct, hence
$\hat R_1$ is also likely to be biased. 
We also report approximate 95\% confidence intervals which are based on 
the usual normal approximation disregarding, among others, any effects of the 
variability in the realization length $T$. 

If the variation in $T$ were merely an observational artifact, then instead of
the above cross-validation averages \eqref{eqn:cvestimate}, we would focus on the total error rate for 
the entire dataset given by \eqref{eqn:overall} below.
\begin{eqnarray}\label{eqn:overall}
 \hat e&=& \frac{\sum\limits_{n=1}^{N}\sum\limits_{t=1}^{T_n} \mathbb{I}_{\{\hat y_t(n)\neq y_t(n)\}}}
{\sum\limits_{n=1}^{N}T_n}=\sum\limits_{n=1}^{N}w(n)\hat e_n,
\quad\text{where}~w(n)=\frac{T_n}{\sum\limits_{n=1}^NT_n}.
\end{eqnarray}
However, to obtain sensible confidence intervals in this setting, we need to estimate the 
variance of $\hat e$.  Bootstrapping is a possibility, but we instead simulate
several (specifically, 15) synthetic datasets in the same fashion as described above,
i.e. re-sampling individual realizations $(x^{T_n}(n),y^{T_n}(n))$
from the HMM with parameters
$\{\hat \pi_{sn}\}_{s\in S}$, $\widehat{\mathbb{P}}_n$, $\{\widehat{P}_{sn}, s\in S\}$, $n=1,2,\ldots,N$. 
We then use the $t$-distribution (on 14
degrees of freedom) to obtain the 95\% margins of error. 

\subsection{Comparison of the accuracy of the Viterbi and PMAP decoders}\label{sec:pmapvsViterbi}
 A histogram of the difference $\hat e(\mathrm{Viterbi},n)-\hat e(\mathrm{PMAP},n)$ between the empirical 
errors \eqref{eq:en} of the Viterbi  and PMAP decoders
is plotted in black in Figure~\ref{fig:histErrVitLessErrPMAP}.
We also observe that in 85.35\% of the CV
rounds the PMAP decoder is more accurate, and in 10.67\%  -- less accurate, than the Viterbi decoder 
(in 3.98\% of the cases the two methods show the same accuracy).   
\begin{figure}[h!]
\begin{center}
 \includegraphics[width=1.1\textwidth]{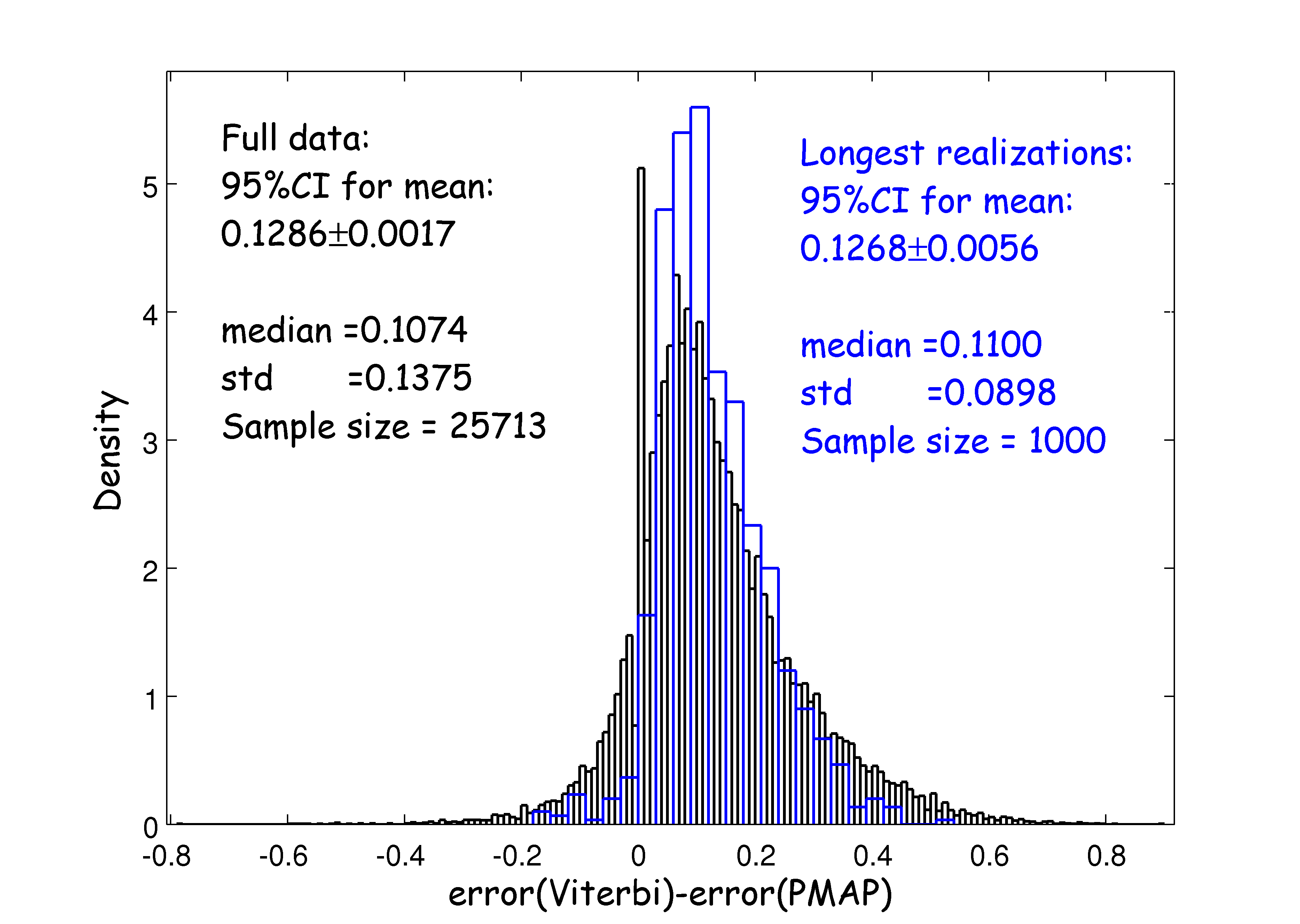}
\end{center}
\caption{A histogram of the difference between the empirical error rates $\hat e(\mathrm{Viterbi},n)
-\hat e(\mathrm{PMAP},n)$ obtained from the full data (black) and the subsample consisting of 
1000 longest realizations (blue). Although in 3.98\% of the entire dataset the two methods show 
the same accuracy (spike at 0), overall their performance appears to be notably different.  
The Viterbi decoder is more accurate in 10.67\% of all the cases, and the PMAP
decoder is more accurate in 85.35\% of all the cases. The extreme differences 
($\min=-78.69\%$,$\max=89.74\%$) tend to be observed on short sequences (136 positions and shorter), 
but the subsample of the 1000 longest realizations (450-2060 positions) 
confirms the effect of the PMAP decoder being more accurate. In particular, on the 
longest sequences, the PMAP decoder can be $52.62\%$ more accurate than the Viterbi
decoder, whereas the latter can be at most $16.75\%$ more accurate than the former.   
}\label{fig:histErrVitLessErrPMAP}
\end{figure} 
To examine sensitivity of these results to the variation in the realization length, 
we superimpose in the same Figure ~\ref{fig:histErrVitLessErrPMAP} a histogram of the 
subsample consisting of the 1000 longest realizations. Although the subsample spans a
less extreme range ($-16.75\%,52.62\%$) than  that of the entire sample,
the locations of the two histograms are very similar, suggesting the average {\it gain of accuracy of 
about 12\% when replacing the Viterbi decoder by the PMAP one}.  

We also compare the performance of 
the Viterbi and PMAP decoders by 
examining their $R_1(\cdot |x^{T(n)}(n))$ risks \eqref{eqn:R1}, see Figure~\ref{fig:histErrVitAndPMAP}.
Note that the difference $\hat R_1(\mathrm{Viterbi}) -\hat R_1(\mathrm{PMAP})$ is 9\% on average,
and is largely unchanged (apart from a minor increase) when
recomputed on the subsample of the 1000 longest realizations (450-2060 positions).
\begin{figure}[h!]
 \centering
 \includegraphics[width=1.1\textwidth,keepaspectratio=true]{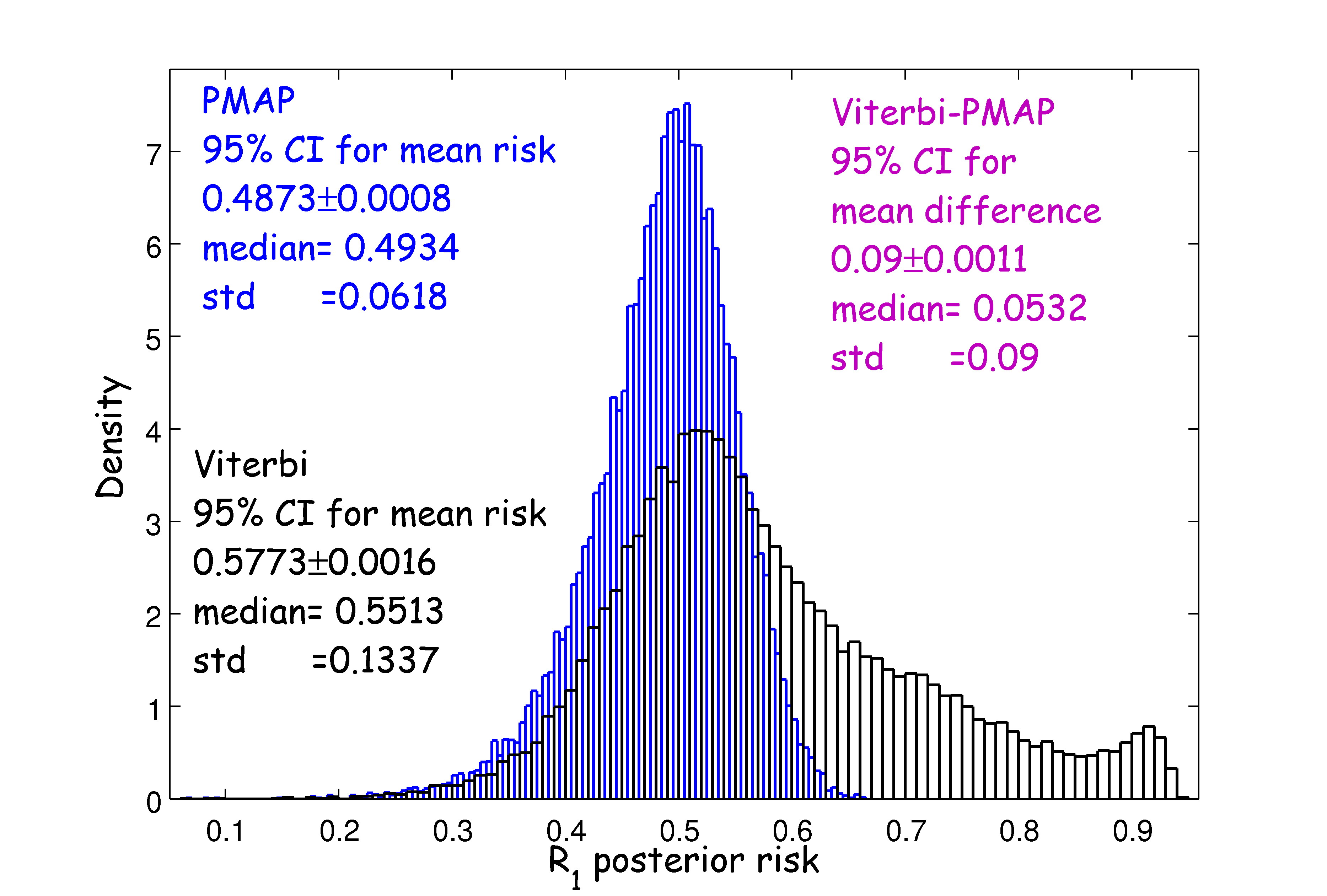}
\caption{Histograms of the $R_1(\hat y^{T(n)}|x^{T(n)}(n))$ risk
of the Viterbi (black) and PMAP (blue) decoders.  Since the first order
homogeneous HMM is only an approximation to the data source, the cross-validation
averages of 48.73\% (PMAP), 57.73\% (Viterbi), and 9\% (PMAP's gain over Viterbi) 
are likely to be biased as estimates of the respective pointwise error rates; 
see also Figure \ref{fig:histErrVitLessErrPMAP} for a model independent analysis.} \label{fig:histErrVitAndPMAP}
\end{figure}

Finally, $\hat e$ \eqref{eqn:overall} is $59.68\%$ ($\pm 0.068\%$) and $46.10\%$ ($\pm 0.047\%$) 
for the Viterbi and PMAP decoders, respectively,
and the PMAP comes out $13.58\%\pm 0.0463\%$ more accurate than the Viterbi decoder.
The above confidence intervals are, however, likely to be deflated since
the model-based simulations show little variation of $\hat e(\mathrm{Viterbi})$, 
$\hat e(\mathrm{PMAP})$, or the differences $\hat e(\mathrm{Viterbi})-
\hat e(\mathrm{PMAP})$. In fact, based on the 15 model-based simulations,  
the PMAP is only $7.46\%\pm 0.0463\%$ more accurate than the Viterbi decoder,
with the individual error rates of  $47.49\%\pm 0.047\%$ and 
$54.95\%\pm 0.068\%$  for the former and the latter, respectively. 
Finally, replacing the empirical error rates by the $R_1(\cdot|x^{T})$
risks (which are now computed exactly since the simulations are model-based), 
we obtain the difference of $8.55\%\pm 0.0213\%$.

{\it In summary, the PMAP decoder can be notably more accurate than the Viterbi decoder in scenarios with 
as few as six hidden states.}

\subsection{The $\R_\infty$ risk  of the Viterbi, PMAP and other 
decoders}\label{sec:pmapvsViterbibyprob}
Next we look at the log-posterior probability 
rates $\log(P(\hat y^T|x^T))/T=-\R_\infty(\hat y^T|x^T)$
of the PMAP, Viterbi and other decoders. In $74.14\%$ of the cases, the PMAP decoder 
returns an inadmissible path, i.e. $\log(P(\hat y^T|x^T))/T=-\infty$. 
To avoid dealing with an infinite range, we switch to the exponential scale. Thus,
Figure \ref{fig:posteriors} below displays histograms of the geometric rates $\sqrt[T]{P(\hat y^T|x^T)}$.
\begin{figure}[h!]
 \centering
 \includegraphics[width=1.0\textwidth]{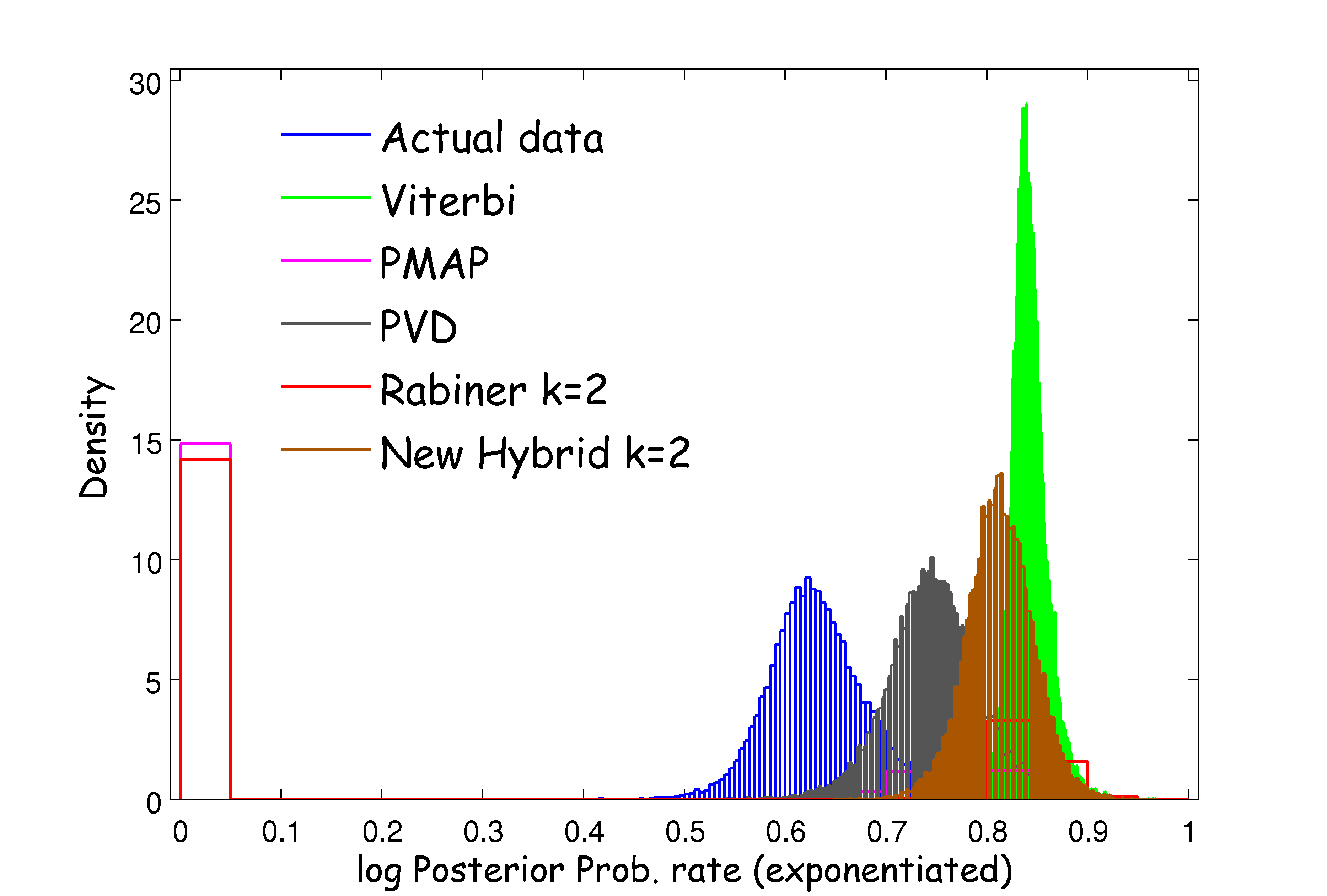}
 \caption{Histograms of the (geometric rates of the) posterior probabilities of the already known 
as well as some new decoders. The Constrained PMAP decoder is omitted as it is virtually
indistinguishable from PVD (gray). The PMAP and Rabiner 2-block (see Subsection \ref{sec:pathdiffers})  decoders 
return inadmissible paths in $74.14\%$ (magenta bar at 0) 
and $70.94\%$ (red bar at 0) of the cases, respectively. Just like the PVD and Constrained PMAP decoders, 
the new hybrid 2-block posterior-Viterbi decoder \eqref{eq:khybrids} (brown histogram) is guaranteed to produce
admissible paths, which would generally have a higher probability than those of the PVD and 
Constrained PMAP paths (at most that of the Viterbi path). } \label{fig:posteriors}
\end{figure}

The Rabiner 2-block decoder $\hat y(2)$ returns inadmissible paths in $70.94\%$ of the cases.
In $7.32\%$ of the cases this decoder gives an inadmissible path even when the PMAP path 
(for the same realization) is admissible.  This illustrates the
violation of monotonicity (see Subsection \ref{sec:pathdiffers}) in the path 
(posterior) probability when using Rabiner's suggestion to base decoding on the loss \eqref{eq:dnew}.  

We also note that the posterior probabilities of the actual hidden paths (blue histogram)  
are notably lower than those of the admissible decodings, especially the Viterbi outputs. 
However, these effects are not out of line with model-based simulations. 

\subsection{Summary of the experimental results}\label{sec:expsummary}

Figure \ref{fig:performance}
compares performance of these  and other decoders as measured by the averaged error rate
and the averaged (exponentiated) path log-posterior rate 
\begin{equation}\label{eqn:cvgeorate}
\widehat{\sqrt[T]{P(\hat y^T|x^T)}}_{CV}=\frac{1}{N}\sum_{n=1}^N\sqrt[T_{n}]{P(\hat y^{T_n}|x^{T_n}(n))}.  
\end{equation}
Recall that the family of k-block posterior-Viterbi decoders is naturally parameterized by the
block length $k$ ($k=1$ and $k\to\infty$ giving the PMAP and Viterbi decoders, 
respectively).  We have also 
included the continuous re-parameterization \eqref{k-blockrec22} 
via $k=\frac{1}{1-\alpha}$ (and $\alpha=\frac{k-1}{k}$) which embeds
these special cases into  the generalized PVD Problem \eqref{gen-problem}
via $C_1=\alpha$, $C_2=1-\alpha$, $C_3=C_4=0$.
   
Figure \ref{fig:performance} displays performance of members of the generalized PVD 
(Problem \eqref{gen-problem}) and PMAP (Problem \eqref{L1-gen-problem})
families with $C_1=\alpha$, $C_2=1-\alpha$, $C_3=C_4=0$ for a subset of values of $\alpha$ 
used in Figure \ref{fig:Case877} and Table \ref{tab:case877}. 
The point-wise maximum a priori ($C_1=C_2=C_4=0$, $C_3=1$) 
and the prior-based Viterbi ($C_1=C_2=C_3=0$, $C_4=1$) decoders are also included, showing identical performance on these data. 
Remarkably, the accuracy of
these ``data-blind'' decoders on average is still higher than that of the Viterbi (MAP) decoder.  
 
\begin{figure}[h!]
 \centering
 \includegraphics[width=1 \textwidth]{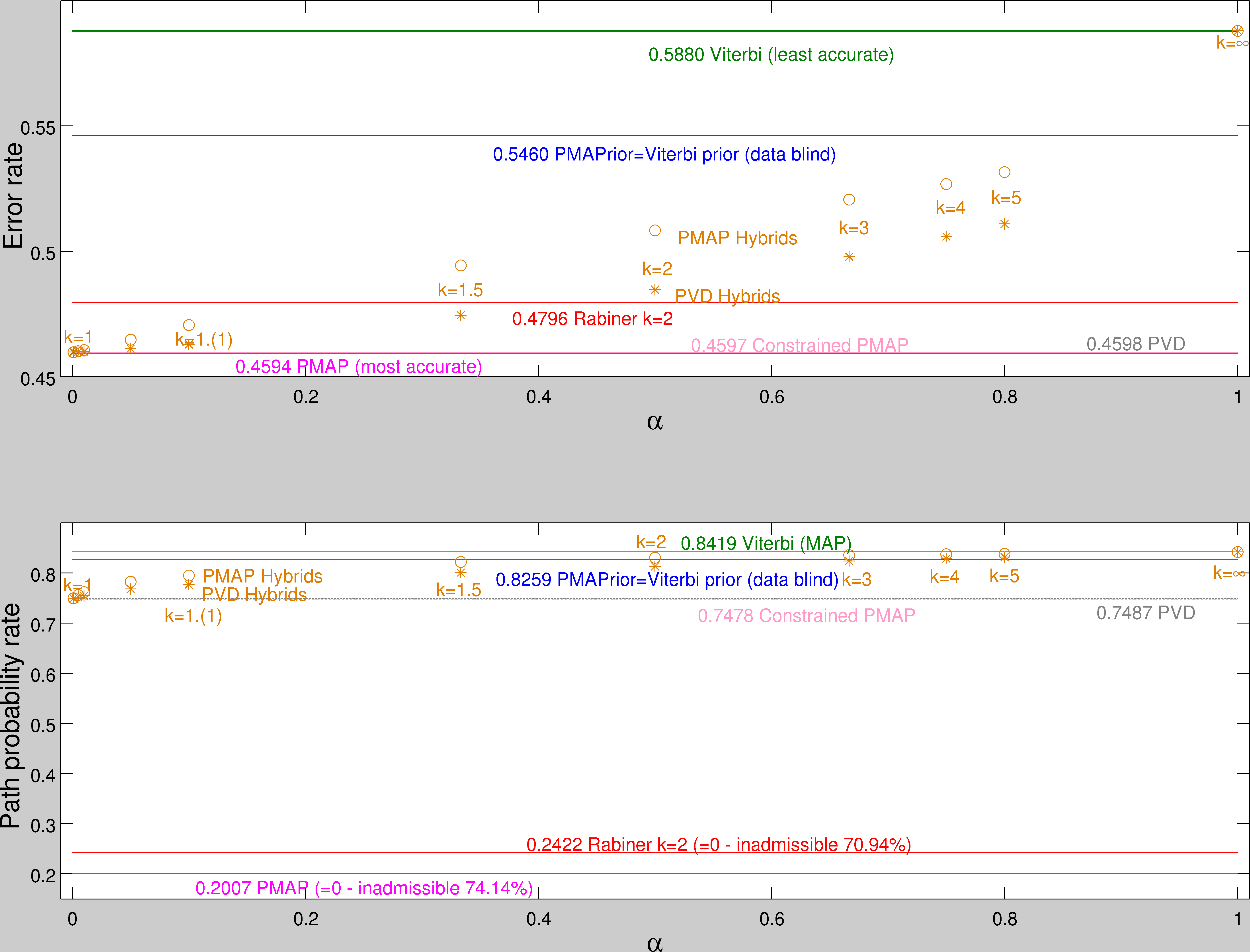}
 \caption{
Empirical error  \eqref{eq:ecv} (top) and probability \eqref{eqn:cvgeorate} (bottom) rates 
of the popular and some new members of the generalized
PVD (asterisk) and PMAP (circle) families. }

\label{fig:performance}
\end{figure}

\section{Asymptotic risks}\label{sec:asy}
Given a classifier $g$ and a risk function (which
it may have been more accurate to call a functional) $R$, the quantity
$R(g(x^T)|x^T)$ evaluates the risk when $g$ is applied
to a given sequence $x^T$. Below we will write $R(x^T)$ for the minimum risk 
$\min_{s^T}R(s^T|x^T)$, which is delivered by the Bayes
classifier $g^*$, i.e. $R(g^*(x^T)|x^T)=R(x^T)$.
Besides $R(X^T)$, we are also interested in
the random variables $R(g(X^T),X^T)$ (depending on $R$ and $g$). Thus, in \citep{kuljus},  
convergence of several risks of the Viterbi path has been considered. 
Since Viterbi paths $v(x^T,\infty)$ and $v(x^{T+1},\infty)$ may differ 
significantly, asymptotic analysis of the Viterbi decoding is far from being trivial. 
In \citep{AVTK2, AVT4, AVTproof},  we constructed  a well-defined process 
$v(X^\infty; \infty)$,  named also after Viterbi, that for a wide class of HMMs extends {\it ad infinitum} finite 
Viterbi paths $v(x^T,\infty)$ and possesses useful ergodic properties. 
Based on the asymptotic theory of Viterbi processes $v(X^\infty; \infty)$, it has
been shown in \citep{kuljus}
that under fairly general assumptions on the HMM, 
the random variables $R_k(v(X^T; \infty)|X^T)$, $\R_k(v(X^T; \infty)|X^T)$, where $k=1,2,\ldots$, and
$\R_{\infty}(v(X^T; \infty)|X^T)$,  as well as 
$\R_{\infty}(v(X^T; \infty))$ (see \eqref{eq:rinftyprior}),   $\R_{1}(v(X^T; \infty))$ (see \eqref{eqn:barR1prior}), 
and $R_1(v(X^T; \infty))$ (see \eqref{eqn:R1prior}) all
converge (as $T\to \infty$) {\it a.s.} to constant (i.e. non-random) limits. 
Convergence of these risks 
implies {\it a.s.} convergence of
$$C_1\R_{1}(v(X^T; \infty)|X^T)+C_2\R_{\infty}(v(X^T; \infty)|X^T)+C_3\R_1(v(X^T;\infty))+C_4\R_{\infty}(v(X^T;\infty)),$$
and
$$C_1R_{1}(v(X^T; \infty)|X^T)+C_2\R_{\infty}(v(X^T; \infty)|X^T)+C_3R_1(v(X^T;\infty))+C_4\R_{\infty}(v(X^T;\infty)),$$
the risks appearing in the generalized problems \eqref{gen-problem} and
\eqref{L1-gen-problem}, respectively. Actually,  convergence of
$\R_{\infty}(v(X^T; \infty), X^T)$ is also proved (and used in the proof of convergence
of $\R_{\infty}(v(X^T; \infty)|X^T)$).
Hence, the minimized risk in \eqref{gen-problem1},
evaluated at the Viterbi paths, converges as well.

The limits -- {\it asymptotic risks} -- are (deterministic) constants that
depend only on the model and help us assess the Viterbi inference in this principled way.
For example, let $R_1(k=\infty)$ be the limit (as $T\to\infty$) of $R_1(v(X^T; \infty)|X^T)$,
which is the asymptotic misclassification rate of
the Viterbi decoding. Thus, for large $T$, the Viterbi decoding makes about 
$TR_1(k=\infty)$ misclassification errors. The asymptotic risks might be, in
principle, found theoretically, but in reality can be rather difficult. However, 
since all these asymptotic results also hold in the $L_1$ sense, 
which implies convergences of expectations, 
the limiting risks can  be estimated by simulations.

In \citep{PMAP, PMAP2}, it has been also shown that under the same assumptions
$R_1(X^T)=R_1(v(X^T; 1)|X^T)$ converges to a constant limit, say
$R_1$. In \citep{kuljus}, 
$\R_1(X^T)=\R_1(v(X^T; 1)|X^T)$ has been also shown to converge. 
Clearly $R_1(k=\infty)\geq R_1(1)$, and even if their difference is
small, the total number of errors made by the Viterbi decoder
in excess of PMAP in the long run can still be significant.
\begin{sloppypar} 
Presently, we are not aware of a universal method  for proving (or improving upon) the
limit theorems for these risks. Recall that convergence of the risks of the Viterbi
decoding is possible due to the existence of the  Viterbi
process which has nice ergodic properties. The question whether infinite PMAP processes
have similar properties, is still open. Therefore, 
convergence of $R_1(X^T)$ 
 was proven with a completely
different method based on the smoothing probabilities. In fact, all of the limit
theorems obtained thus far have been proven with different methods. 
We conjecture that these different methods can be combined so that
convergence of the minimized combined risk \eqref{gen-problem} or
\eqref{L1-gen-problem} could be proven as well. In summary, as mentioned
before,  convergence of the minimized combined risks 
has thus far been obtained for trivial combinations only, i.e. 
with three  of the four constants being zero.  Note that while
convergence of the intermediate case \eqref{k-blockrec22}
with its minimizer $v(x^T; k(\alpha))$
is an open question, \eqref{vorratused} gives
\end{sloppypar}
\begin{equation*}
0\le \R_{\infty}(v(x^T; k(\alpha))|x^T)-\R_{\infty}(v(x^T; \infty)|x^T) \le {\R_1(v(x^T; \infty)|x^T)\over k-1}. 
\end{equation*}
This, together with the {\it a.s.} convergence of $\R_1(v(X^T; \infty)|X^T)$, 
implies that in the long run, for most sequences $x^T$, 
$\R_{\infty}(v(x^T; k)|x^T)$ will not exceed $\R_{\infty}(v(x^T; \infty)|x^T)$ 
by more than $\frac{1}{k-1}\lim_{T\to\infty}\R_1(v(X^T; \infty)|X^T)$.
Since this limit is finite, letting $k$ 
increase with
$T$,  we have $\R_{\infty}(v(X^T;k_T))$ approach $\lim_{T\to \infty}
\R_{\infty}(v(X^T;\infty))$ {\it a.s.}, i.e. as the intuition predicts, the
likelihood of   $v(X^T;k_T)$ approaches that of  $v(X^T;\infty)$.

Finally, in \citep{AVTproof, intech} we also outline possible applications of the above asymptotic
risk theory.  For example, if a certain number of the true labels $y_1, y_2, \ldots, y_T$ can be revealed 
(say, at some cost), the remaining labels would be computed by a constrained decoder, e.g.
the constrained Viterbi decoder. Having observed $x^T$, the user then needs to decide which positions
are ``most informative'' and then acquires their labels.  Assuming further that the HMM
is stationary, the $R_1$-like risks $P(v(X^\infty; \infty)_t \neq Y_t| X_{t-m}^{t+m}\in A)$ (for any $m\ge 1$
and any measurable set $A\in \mathcal{X}^{2m+1}$), are independent of $t$ (for $t=m+1,m+2,\ldots$), and 
could therefore be used  in the above active learning protocol for the selection of the most informative
 positions.
Specifically, if $A$ is such that $P(v(X^\infty; \infty)_t \neq Y_t| X_{t-m}^{t+m}\in A)$ is high,
then acquire labels at positions $t$ of occurrence of $A$.  Naturally, there are different ways to make
this concrete. For one simple example, suppose only a batch of $L$ labels can be acquired. 
Assuming $\mathcal{X}$ is discrete, order all the $\mathcal{X}$ words $A$ of length $q$ (i.e. $A\in \mathcal{X}^q$)
by $P(v(X^\infty; \infty)_t \neq Y_t| X_{t-m}^{t+m}\in A)$. Finally, from the 
$\mathcal{X}$ of length $q$ that occur in $x^T$, choose $L$ with the highest $P(v(X^\infty; \infty)_t \neq Y_t| X_{t-m}^{t+m}\in A)$.
The above asymptotic theory is crucial also for establishing 
$P(v(X^\infty; \infty)_t \neq Y_t| X_{t-m}^{t+m}\in A)$
as the {\it a.s.} limit of easily computable (e.g. via off-line simulations) empirical measures.  
In practice, these latter measures
would be used as estimates of $P(v(X^\infty; \infty)_t \neq Y_t| X_{t-m}^{t+m}\in A)$ and first experiments
along these lines are given in \citep[4.4]{intech}.  It is also of interest to test these ideas with 
other risks and decoders, such as members of the generalized hybrid families presented here.


\section{Discussion}\label{sec:discussion}
The point-wise symmetric zero-one loss $l(y,s)=\mathbb{I}_{\{s\neq y\}}$ in \eqref{point-loss}, \eqref{eqn:R1},
and consequently in the generalized PMAP hybrid decoding \eqref{L1-gen-problem}, 
can be easily replaced by a general loss $l(y,s)\ge 0$, $s,y\in S$ 
(assuming, without loss of generality, $\sum_{s,y\in S} l(y,s)=1$). In computational terms, 
this would require multiplying the loss matrix $(l(y,s))_{y,s\in S}$ by 
the (prior or) posterior probability vectors 
$(p_t(1|x^T),p_t(2|x^T),\ldots,p_t(K|x^T))^{'}$ to obtain the (prior or) posterior
risk $(R_t(1|x^T),R_t(2|x^T),\ldots,R_t(K|x^T))^{'}$ vectors (we use the apostrophe to
denote vector transpose).  The dynamic programming algorithm defined by \eqref{L1-gen-rec}
still stands provided $p_t(j|x^T)$ (or $p_t(j)$, or both) is replaced by $1-R_t(j|x^T)$
(or $1-R_t(j)$, or both respectively) in the definition of $\gamma_t(j)$. 
If all confusions of state $y$ are equally undesirable, i.e. $l(y,s)$ is of the form $l(y)\times \mathbb{I}_{\{s\neq y\}}$, 
then the above adjustment reduces to replacing $p_t(j|x^T)$ by $l(j)p_t(j|x^T)$ (for all $j\in S$).  

Using an asymmetric loss could be particularly valuable in practice when, for example, 
detection of a rare state or transition needs to be boosted. Similar views have been most recently expressed 
also in \citep{YauHolmes2010arXiv},
who, staying within the additive risk framework, have proposed a general asymmetric
form of the loss \eqref{eq:dnew} with $k=2$.  Hybridizing this general asymetric pairwise loss 
with the other losses considered in this work should provide additional flexibility to path
inference. A way to incorporate this loss into our generalized framework is by vectorizing
the chain $\{Y_t\}_{t\ge 1}$ as $\{(Y_t,Y_{t+1})\}_{t\ge 1}$ and then following the opening  lines of
this Section. 

Also, using a range
of perturbed versions of a loss function can help assess saliency of particular detections (``islands''). 
In fact, at the stage of data exploration
one may more generally want to use a collection of outputs produced by using
a range of different loss functions instead of a single one.

The logarithmic risks \eqref{loglikerisk}, \eqref{eqn:barr1}, \eqref{eq:rinftyprior}, 
\eqref{eqn:barR1prior} on the one hand, 
and the ordinary risks \eqref{eqn:Rinf}, \eqref{eqn:R1}, $R_\infty(s^T)=1-p(s^T)$, 
\eqref{eqn:R1prior}, on the other hand, can be respectively combined into a 
single parameter family of risks by using, for example,
the power transformation as shown below with $p$ for the moment standing for any probability distribution on $S^T$.  
\begin{eqnarray}\label{eqn:powerrisks}
 R_1(s^T; \beta)&=& \begin{cases}
                   - \frac{1}{T}\sum_{t=1}^T \frac{p_t(s_t)^\beta-1}{\beta}, &\text{if}~\beta\neq 0\\
                   - \frac{1}{T}\sum_{t=1}^T \log p_t(s_t) &\text{if}~\beta=0 
                  \end{cases}\\ \nonumber
R_\infty(s^T; \beta)&=& \begin{cases}
                   - \frac{1}{T}\frac{p(s^T)^\beta-1}{\beta}, &\text{if}~\beta\neq 0\\
                   - \frac{1}{T}\log p(s^T) &\text{if}~\beta=0 
                  \end{cases}
\end{eqnarray}
Thus,  the family of risk minimization problems  given in \eqref{supergen-problem} below
\begin{equation}\label{supergen-problem}
\min_{s^T}\Big[C_1{R}_1(s^T|x^T; \beta_1)+C_2{R}_{\infty}(s^T|x^T; \beta_2)+C_3{R}_1(s^T; \beta_3)+C_4{R}_{\infty}(s^T; \beta_4)\Big],
\end{equation}
 $C_i\ge 0$ and $\sum_{i=1}^4 C_i>0$ unifies and generalizes problem \eqref{gen-problem} ($\beta_1=\beta_2=\beta_3=\beta_4=0$)
and problem \eqref{L1-gen-problem} ($\beta_1=\beta_3=1$, $\beta_2=\beta_4=0$).
Clearly, the dynamic programming approach of Theorem~\ref{dyn} and \eqref{L1-gen-rec}
immediately applies to any member of the above family \eqref{supergen-problem} 
with $\beta_2=\beta_4=0$. Also, computations of multiple decoders from this 
family (at least with $\beta_2=\beta_4=0$) are readily parallelizable. 

Next, Theorem \ref{k-block} and Corollaries \ref{corollary0} and \ref{corollary} obviously
generalize to higher order Markov chains as can be seen from the following Proposition.
\begin{proposition}\label{prop:generalMC}
Let $p$ represent a Markov chain of order $m$,
  $1\le m\le T$, on $S^T$.
Then for any $s^T\in S^T$ and for any $k\in\{m,m+1,\ldots\}$, 
we have
$$\bar R_k(s^T)=\bar R_m(s^T)+(k-m)\bar R_\infty(s^T).$$
\end{proposition}
\begin{proof}
This is a straightforward extension of the proof of Theorem \ref{k-block}.
\end{proof}

The present risk-based discussion of HMM path inference also naturally extends
to the problem of optimal {\it labeling} or {\it annotation} (already mentioned in subsection \ref{sec:segment}).
Namely, the state space $S$ can be partitioned into subsets $S_1$, $S_2$, \ldots, $S_\Lambda$, for
some $\Lambda\le K$, in which case $\lambda(s)$ assigns label $\lambda$ to every state $s\in S_\lambda$. 
The fact that the PMAP problem is as easily solved over the label space $\Lambda^T$ as it is 
over $S^T$ has already been used in practice. Indeed, 
\citep{KallKrogh2005}, who also add the constraint of admissibility with respect to the prior distribution,
in effect average $p_t(s_t|x^T)$'s, for each $t$, within the label classes 
and then use recursions \eqref{rec1b} to obtain the {\em optimal accuracy  labeling} 
of  {\it a priori} admissible state paths. This  clearly corresponds to using the point loss 
$l(s,s')=\mathbb{I}_{\{\lambda(s)\ne \lambda(s')\}}$ in \eqref{point-loss}
when solving $\min_{s^T: p(s^T)>0}R_1(s^T|x^T)$ \eqref{eq:remove0s2}.  With our definition 
of admissibility (i.e. positivity of the posterior path probability), the same approach
(i.e. replacing $p_t(s_t|x^T)$'s by their within class average $\bar p_t(s_t|x^T)$)
extends to solve $\min_{s^T: p(s^T|x^T)>0}R_1(s^T|x^T)$ \eqref{eq:remove0s} under the same
loss $l(s,s')=\mathbb{I}_{\{\lambda(s)\ne \lambda(s')\}}$.
Clearly, the generalized problem \eqref{supergen-problem} also immediately
incorporates the above pointwise label-level loss in either the prior $R_1(\cdot; \beta_3)$ or posterior
risk $R_1(\cdot; \beta_1)$, or both. 
Since computationally these problems are essentially 
as light as  \eqref{L1-gen-rec} and since \citep{KallKrogh2005} report
their special case to be successful in practice, we believe that the above
generalizations offer yet more  possibilities that are potentially useful in practice.

Instead of using the same arithmetic averages $\bar p_t(s_t|x^T)$'s (or $\bar p_t(s_t)$'s)
for the $R_1$ risks in \eqref{supergen-problem} regardless
of $\beta$, we can gain additional flexibility by replacing $\bar p_t(s_t)^\beta$ and 
$\log  \bar p_t(s_t)$ in \eqref{eqn:powerrisks} ($\beta\neq 0$ and $\beta=0$ respectively)
with
\begin{eqnarray*}
 \bar p_t(s; \beta)&\propto&
\begin{cases}
\left(\frac{\sum\limits_{s'\in S_{\lambda(s)}}p_t(s')}{|S_{\lambda(s)}|}\right)^\beta, &\text{if}~\beta\neq 0,\\
\left(\prod\limits_{s'\in S_{\lambda(s)}}p_t(s')\right)^\frac{1}{|S_{\lambda(s)}|}, &\text{if}~\beta= 0.
\end{cases}
\end{eqnarray*}

Certainly, the choice of the basic loss functions,  inflection parameters 
$\beta_i$ and weights $C_i$ of the respective risks, is application
dependent, and can be tuned with the help of labeled data, using, for example, 
cross-validation. Finally, these generalizations are presented for the standard
HMM setting, and therefore extensions to more complex and practically more
useful HMM-based settings (e.g semi-Markov, autoregressive, etc.) could naturally
become of interest next.

\section*{Acknowledgement}
The first author has been supported  by the Estonian Science Foundation Grant nr. 9288 
and by targeted financing project SF0180015s12, which has also supported a research
visit of the second author to Tartu University. The second author has also been 
supported by UK NIHR Grant i4i II-AR-0209-10012.  The authors are also grateful to anonymous 
reviewers as well as to the action editor for their thorough reviews of this work, 
additional references, and comments and suggestions on improving this manuscript.
The authors are also very thankful to Dr Dario Gasbarra for reviewing an earlier version 
of the manuscript and pointing out a subtle mistake, as well as to Ufuk Mat for pointing
out some typing errors. 
\appendix
\section{An example of an inadmissible path of positive prior probability}
\label{sec:appendixposprior0posterior}
\begin{eqnarray*}
\pi&=&\begin{pmatrix}1 & 1 & 1 & 1 & 1 & 1 & 1 & 1 & 1 \end{pmatrix}/9, \\
P&=& \begin{pmatrix}
   5  &  0 & &  0 & 0  & 4 & 0 &  0 & 0 &  0 \\     
   1  &  1 & &  1 & 1  & 1 & 1 &  1 & 1 &  1 \\     
   0  &  0 & &  4 & 0  & 5 & 0 &  0 & 0 &  0 \\     
   1  &  1 & &  1 & 1  & 1 & 1 &  1 & 1 &  1 \\     
   0  &  1 & &  0 & 1  & 1 & 1 &  2 & 1 &  2 \\     
   1  &  1 & &  1 & 1  & 1 & 1 &  1 & 1 &  1 \\     
   3  &  3 & &  0 & 0  & 0 & 3 &  0 & 0 &  0 \\     
   1  &  1 & &  1 & 1  & 1 & 1 &  1 & 1 &  1 \\     
   0  &  0 & &  3 & 3  & 0 & 0 &  0 & 3 &  0
\end{pmatrix}/9.
 \end{eqnarray*}
To simplify the verifications, consider an emission alphabet
with only four symbols, although the idea of constructing this example readily 
extends to more alphabets with more states (in particular, to more practically
relevant situations where the emission alphabet is bigger that the hidden state
space or the emission distributions are continuous altogether). 
Then take the following emission distributions:
\begin{eqnarray*}
  \begin{pmatrix} 
    P_1     &     P_2 &  P_3& P_4 \\
      1/25  &        1/20  &         0   &        91/100   \\
       0    &         0    &        1/5  &         4/5     \\
      1/20  &        1/25  &         0   &        91/100   \\
       0    &         0    &        1/5  &         4/5     \\
      1/10  &         0    &        1/5  &         7/10    \\
       0    &         0    &        1/5  &         4/5     \\
      1/15  &        1/15  &         0   &        13/15    \\
       0    &         0    &        1/5  &         4/5     \\
      1/15  &        1/15  &         0   &        13/15 
 \end{pmatrix}.
\end{eqnarray*}
Suppose now that a sequence $x^3=(1,2,3)$ has been observed. It can then be 
verified that the (unconstrained) PMAP decoder returns any of the following
paths $(5,1,5)$, $(5,3,5)$, $(5,7,5)$, or $(5,9,5)$, all of which having
zero prior (and posterior) probabilities.   

When the decoder is subject
to the positivity constraint on the prior probabilities, it would return 
any of the following paths $(5,2,5)$, $(5,4,5)$,  $(5,5,5)$, $(5,6,5)$, 
$(5,8,5)$, which, despite being of positive prior probabilities, all have 
zero posterior probabilities. 

Finally, if the decoder is constrained
to produce paths of positive posterior probability, it would then return
any of the following paths $(5,7,2)$, $(5,7,6)$, $(3,3,5)$, $(9,3,5)$.

\section{Further details of the experiments from Section \ref{sec:example}}
\label{sec:appendix1} 
Below are the estimates of the HMM parameters obtained from the entire dataset
as described in Section \ref{sec:example}. 
\begin{eqnarray*}
& \hat \pi=& \begin{pmatrix}0.0016 & 0.0041 & 0.9929 & 0.0014 & 0.0000 & 0.0000 \end{pmatrix}, \\
\begin{matrix}
   1\\2\\3\\4\\5\\6
  \end{matrix}  &
 \widehat{\mathbb{P}}=&  \begin{pmatrix}
    0.8359 &   0.0034 &   0.1606 &        0 &        0 &        0    \\
    0.0022 &   0.8282 &   0.1668 &   0.0028 &        0 &        0    \\
    0.0175 &   0.0763 &   0.8607 &   0.0455 &        0 &        0    \\
         0 &        0 &        0 &   0.7500 &   0.2271 &   0.0229    \\
         0 &        0 &        0 &        0 &   0.8450 &   0.1550    \\
         0 &   0.0018 &   0.2481 &        0 &        0 &   0.7501    \end{pmatrix},\\
& \hat \pi_{inv}=& \begin{pmatrix}0.0511&    0.2029 &   0.4527&    0.0847 &   0.1240&    0.0847\end{pmatrix}, \\
\begin{matrix}
   A\\C\\D\\E\\F\\G\\H\\I\\K\\L\\M\\N\\P\\Q\\R\\S\\T\\V\\W\\Y
  \end{matrix}  
 & &
  \begin{pmatrix} 
    \widehat{P_1}     &     \widehat{P_2} &  \widehat{P_3}& \widehat{P_4}& \widehat{P_5} & \widehat{P_6} \\
    0.1059  &  0.0636 &   0.0643  &  0.1036 &   0.1230 &   0.1230  \\
    0.0107  &  0.0171 &   0.0135  &  0.0081 &   0.0111 &   0.0128  \\
    0.0538  &  0.0319 &   0.0775  &  0.0634 &   0.0415 &   0.0345  \\
    0.0973  &  0.0477 &   0.0620  &  0.1120 &   0.0852 &   0.0848  \\
    0.0436  &  0.0576 &   0.0330  &  0.0371 &   0.0386 &   0.0399  \\
    0.0303  &  0.0484 &   0.1133  &  0.0447 &   0.0321 &   0.0229  \\
    0.0203  &  0.0227 &   0.0259  &  0.0188 &   0.0197 &   0.0221  \\
    0.0564  &  0.1010 &   0.0372  &  0.0557 &   0.0694 &   0.0593  \\
    0.0672  &  0.0443 &   0.0574  &  0.0560 &   0.0671 &   0.0810  \\
    0.1227  &  0.1068 &   0.0674  &  0.0994 &   0.1279 &   0.1477  \\
    0.0240  &  0.0219 &   0.0181  &  0.0214 &   0.0293 &   0.0304  \\
    0.0299  &  0.0252 &   0.0561  &  0.0259 &   0.0338 &   0.0336  \\
    0.0333  &  0.0208 &   0.0757  &  0.0472 &   0.0067 &   0.0031  \\
    0.0443  &  0.0270 &   0.0330  &  0.0469 &   0.0497 &   0.0472  \\
    0.0594  &  0.0464 &   0.0470  &  0.0522 &   0.0677 &   0.0697  \\
    0.0496  &  0.0496 &   0.0744  &  0.0485 &   0.0422 &   0.0491  \\
    0.0395  &  0.0641 &   0.0572  &  0.0465 &   0.0412 &   0.0375  \\
    0.0591  &  0.1386 &   0.0473  &  0.0685 &   0.0677 &   0.0545  \\
    0.0168  &  0.0172 &   0.0111  &  0.0135 &   0.0130 &   0.0124  \\
    0.0359  &  0.0483 &   0.0286  &  0.0306 &   0.0332 &   0.0344               
\end{pmatrix}.
\end{eqnarray*}

\bibliographystyle{amsplain}
\bibliography{../ref,../../Microimage/ref,../../eur/ref,../../DWI/DTIMIX/ref,../Apples/ref}

\end{document}